\pdfoutput=1
\documentclass[11pt]{article}

\PassOptionsToPackage{dvipsnames,table}{xcolor}

\usepackage[final]{acl}

\usepackage{times}
\usepackage{latexsym}
\usepackage[T1]{fontenc}
\usepackage[utf8]{inputenc}
\usepackage{microtype}
\usepackage{inconsolata}
\usepackage{graphicx}

\usepackage{amsmath,amsfonts,bm}

\def\eqref#1{equation~\ref{#1}}

\def\1{\bm{1}}

\def\vb{{\bm{b}}}

\def\vh{{\bm{h}}}

\def\vs{{\bm{s}}}

\def\vx{{\bm{x}}}
\def\vy{{\bm{y}}}
\def\vz{{\bm{z}}}

\def\mH{{\bm{H}}}

\def\mT{{\bm{T}}}

\def\mW{{\bm{W}}}
\def\mX{{\bm{X}}}

\def\mZ{{\bm{Z}}}

\DeclareMathAlphabet{\mathsfit}{\encodingdefault}{\sfdefault}{m}{sl}
\SetMathAlphabet{\mathsfit}{bold}{\encodingdefault}{\sfdefault}{bx}{n}

\def\gG{{\mathcal{G}}}

\def\gL{{\mathcal{L}}}

\def\gO{{\mathcal{O}}}

\def\sD{{\mathbb{D}}}

\def\sR{{\mathbb{R}}}

\newcommand{\parents}{Pa} 

\usepackage{subcaption}
\usepackage{booktabs}
\usepackage{pifont}
\usepackage{mathtools}
\usepackage{amsmath}
\usepackage{amssymb}
\usepackage{amsthm}
\usepackage{bm}
\usepackage{multirow}

\newcommand{\yhlt}[1]{#1}
\newcommand{\ghlt}[1]{#1}

\newcommand{\mn}[1]{\texttt{CircuitLasso#1}}

\definecolor{aliceblue}{rgb}{0.94, 0.97, 1.0}

\usepackage[capitalize,noabbrev]{cleveref}

\theoremstyle{plain}
\newtheorem{theorem}{Theorem}[section]
\newtheorem{proposition}[theorem]{Proposition}

\theoremstyle{definition}

\theoremstyle{remark}

\title{Scalable Circuit Learning for Interpreting Large Language Models}

\author{
  \textbf{Naiyu Yin\textsuperscript{1}},
  \textbf{Dennis Wei\textsuperscript{2}},
  \textbf{Tian Gao\textsuperscript{2}},
  \textbf{Amit Dhurandhar\textsuperscript{2}},
\\
  \textbf{Karthikeyan Natesan Ramamurthy\textsuperscript{2}},
  \textbf{Yue Yu\textsuperscript{1}}
\\
\\
  \textsuperscript{1}Department of Mathematics, Lehigh University, Bethlehem, PA, USA
  \\
  \textsuperscript{2}IBM Research, Yorktown Heights, NY, USA
\\
  \small{
    \textbf{Correspondence:} \href{mailto:nay224@lehigh.edu}{nay224@lehigh.edu}
  }
}

\begin{document}
\maketitle

\begin{abstract}
    A prominent research direction in mechanistic interpretability is learning sparse circuits over LLM components to reveal how they jointly produce model behavior.  However, raw neurons are polysemantic, making learned circuits hard to interpret.  Sparse autoencoder (SAE) features alleviate this, but their high dimensionality makes existing intervention-based circuit learning methods computationally prohibitive.  We propose \mn{}, a scalable circuit learning approach based on sparse linear regression.  \mn{} recovers circuits whose structural accuracy matches that of state-of-the-art intervention-based methods on the benchmark data, at a fraction of the computational cost.  For interpretability, \mn{} efficiently uncovers relationships among SAE features, showing how human-interpretable semantic features propagate through the model and influence its predictions.  Finally, we validate the utility of our learned circuits by leveraging their insights to achieve comparable performance at substantially lower cost on a domain-generalization task.
\end{abstract}

\section{Introduction}
The fundamental challenge of mechanistic interpretability is to understand the ``why'' behind the behaviors of large language models (LLMs). A key technique involves discovering circuits, which are compact subgraphs connecting key components within the model (such as attention heads and neurons) that drive a specific behavior or capability~\citep{meng2022locating,geva2023dissecting,nanda2023attribution}. However, existing methods for circuit learning often face a
bottleneck. The raw components of an LLM, such as individual neurons, are known to be polysemantic~\citep{elhage2022toy}, meaning that a single neuron can be activated by and contribute to multiple, seemingly unrelated concepts. This polysemanticity renders the learned circuits dense, noisy, and challenging for humans to interpret, undermining the very goal of mechanistic interpretability.

The limitations of using raw, polysemantic neurons have motivated a shift toward a more promising foundation for circuit analysis based on Sparse Autoencoders (SAEs) and related tools~\citep{bricken2023monosemantic,huben2024sparse}. SAEs are neural networks trained to reconstruct the activations of an LLM's raw components using a high-dimensional but sparse set of ``features''. Remarkably, these SAE features tend to be monosemantic, i.e., each feature consistently activates for a single, human-interpretable concept, such as ``related to sports,'' ``a specific emotion,'' or ``a particular grammatical structure.'' The monosemanticity of SAE features has the potential not only to enhance interpretability in itself but also to yield sparser, cleaner structural graphs, and, we conjecture, more faithful representations of the model's internal processing.

\begin{figure*}[hpt]
    \centering
    \includegraphics[width=0.85\linewidth]{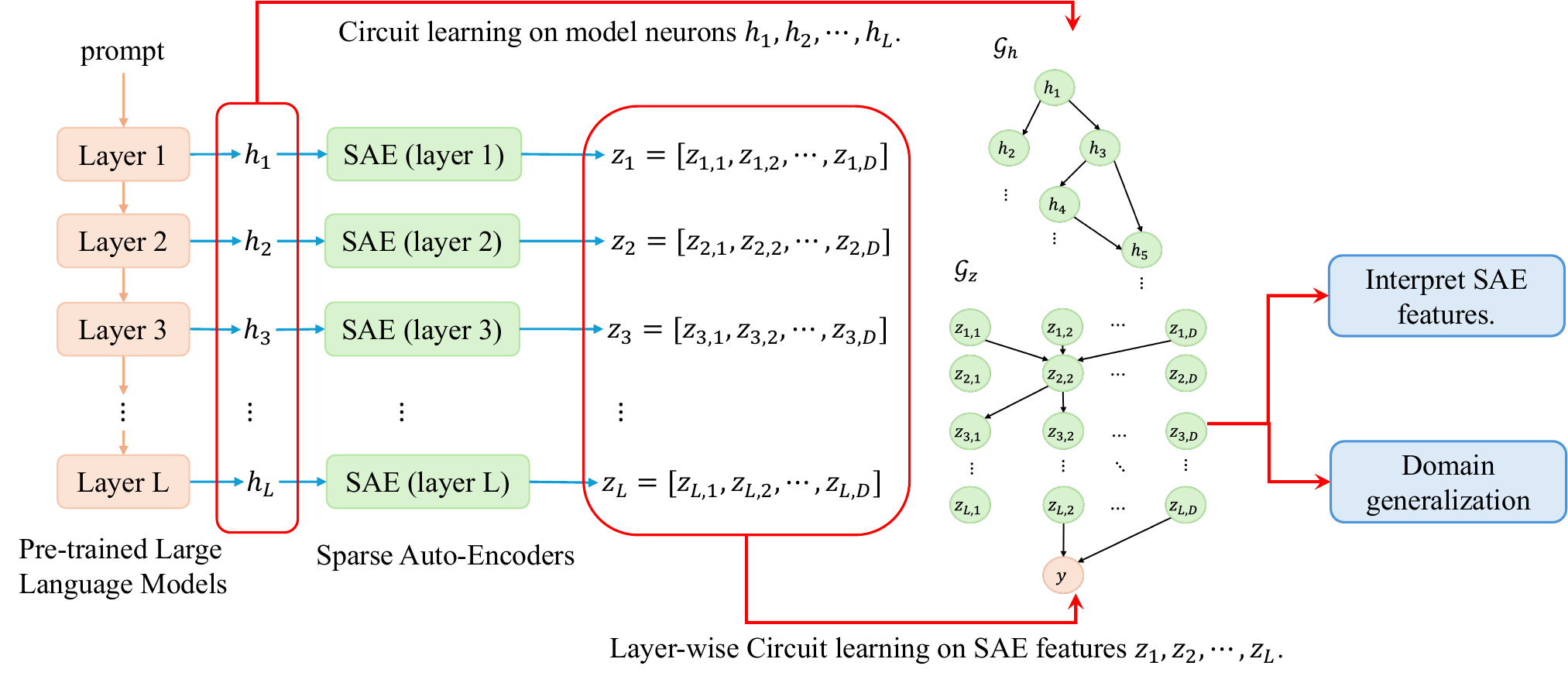}
    \vspace{-0.5cm}
    \caption{An illustration of our model neuron activation and SAE feature collection procedure, learned circuits, and potential downstream tasks.}
    \label{fig:pipeline}
    \vspace{-0.5cm}
\end{figure*}
Our work is motivated by the above
potential of
SAE features to achieve greater
interpretability in LLM circuit analysis. However, existing circuit learning methods~\citep{vig2020causal,meng2022locating,nanda2023attribution,syed2024attribution,kramar2024atp,hanna2024have}, many of which are designed for the lower-dimensional space of raw neurons, struggle to scale
to the high-dimensional feature space of SAEs. The computational complexity and the risk of finding spurious correlations increase dramatically. To address this, we
introduce a novel approach to handle the high dimensionality. Our method, \mn{}, utilizes the Lasso (i.e., $\ell_1$-penalized linear regression) to find a sparse set of connections between features that explains the model's behavior. Sparse linear regression is well-suited for high-dimensional data, as it is computationally efficient and the sparsity translates to more interpretable circuits.

An advantage of our approach is its use of observational data only. This broadens its applicability, and addresses the scalability issue of the existing intervention-based approaches, whose cost scales with LLM size. 
We quantify the efficiency advantage of our regression-based method through a theoretical analysis of its computational cost compared to state-of-the-art intervention-based approaches, establishing conditions under which our method guarantees greater efficiency.

We empirically evaluate \mn{} against state-of-the-art baselines on a circuit learning benchmark, demonstrating substantial improvements in efficiency while accurately capturing circuits involving LLM components. Leveraging its scalability, we then apply our method to SAE features and obtain human-interpretable circuits in a qualitative case study on the CoLA dataset. Finally, leveraging insights from the learned circuits, we achieve comparable accuracy at substantially lower cost on a downstream domain-generalization task.

Our primary contributions are as follows: 1) Inspired by the continuous formulation of causal discovery, we propose a sparse regression surrogate for circuit discovery, \mn{}, and theoretically analyze its computational cost compared to intervention-intensive approaches.
2) \mn{} facilitates operating on monosemantic but high-dimensional SAE features, potentially offering clearer explanations of how human-interpretable concepts propagate through LLMs. 3) Extensive experiments across LLMs (up to 9B parameters), SAEs, and datasets demonstrate that \mn{} achieves efficiency at parity of accuracy on circuit discovery benchmarks and yields useful insights for a downstream generalization task.

\section{Related Work}
\textbf{Mechanistic Interpretability.} Mechanistic interpretability has largely explained behavior in terms of raw or coarse-grained components: induction heads for in-context learning~\citep{olsson2022context}, MLP modules for factual recall~\citep{meng2022locating,geva2023dissecting,nanda2023fact}. Because such components are polysemantic~\citep{elhage2022toy}, the resulting circuits are dense and hard to read, and hypothesis-driven alternatives~\citep{geiger2023causal,zou2023representation} require knowing in advance what to look for. Sparse autoencoders~\citep{bricken2023monosemantic,huben2024sparse} instead recover monosemantic features by dictionary learning, with many training variants since~\citep{rajamanoharan2024jumping,gao2024scaling,bussmann2024batchtopk,dunefsky2024transcoders}. The resulting dictionaries are large, and scaling to them is the obstacle this paper addresses.

\textbf{Intervention and Attribution.} Most circuit discovery estimates influence by intervening: causal mediation~\citep{vig2020causal,geva2023dissecting,hanna2024have}, causal tracing~\citep{meng2022locating}, and activation patching~\citep{nanda2023attribution,syed2024attribution} substitute a counterfactual activation and measure the effect on the output. One intervention per candidate edge makes cost grow with the component set. Gradient approximations, attribution patching~\citep{syed2024attribution}, its integrated-gradient variant~\citep{hanna2024have}, and the relevance-propagation refinement of \citet{jafari2025relp}, lower the constant but keep forward and backward passes in the inner loop. \citet{markssparse} reach SAE features only with heuristic pre-processing such as clustering.

\textbf{Subgraph Search.} A second family searches over subgraphs. \citet{conmy2023towards} prune edges iteratively, echoing constraint-based causal discovery such as the PC algorithm~\citep{pearl2000models}; \citet{bhaskar2024edgepruning} make the search continuous, learning gates by gradient descent and scaling to CodeLlama-13B; \citet{yu2025discogp} prune edges and weights jointly in DiscoGP. These share our premise that per-edge interventions dominate cost, but mask learning is non-convex, keeps backpropagation in the inner loop, and needs one parameter per candidate edge, quadratic in dictionary size over SAE features.

\textbf{Structure from Observation.} Furthest from intervention, \citet{laptev2025analyze} build graphs data-free from SAE decoder weights. \mn{} instead estimates structure from observational activations: one forward pass to cache them, then a convex Lasso, with no backward pass. It draws on the continuous formulation of causal discovery, largely unexplored here, and optimizes a structural rather than behavioral criterion. Masking maximizes faithfulness directly, so it offers guarantees we do not, and an empirical comparison is left to future work. Like other data-driven methods, \mn{} returns one population-level circuit per task, complementary to per-prompt attribution graphs~\citep{ameisen2025circuit,lindsey2025biology}.

\section{Circuit Learning Framework and Methodology}
\subsection{Circuit Discovery via Sparse Linear Regression}\label{sec:general_formulation}
State-of-the-art approaches quantify the importance of hidden representations or computational graph edges by estimating their causal effects, particularly indirect effects, using techniques such as causal mediation analysis~\citep{vig2020causal}, causal tracing~\citep{meng2022locating}, attribution patching~\citep{nanda2023attribution,syed2024attribution}, and related variants~\citep{kramar2024atp,hanna2024have}. These approaches share some similarities with constraint-based causal discovery, which assesses potential edges among variables via independence tests and retains those with strong dependencies. However, constraint-based causal discovery methods are known to face scalability challenges, and circuit discovery methods share this limitation since they must separately quantify the importance of every hidden representation and edge in the computational graph, which can quickly become infeasible with larger models.

We adopt a sparse linear regression framework, structurally analogous to a linear SEM in the continuous causal discovery literature, as a tractable surrogate for the LLM's nonlinear computational graph. \textbf{The goal of this surrogate is not to recover fine-grained per-edge causal effects of the underlying nonlinear computation, but to identify, at scale, the dependency skeleton of the circuit, i.e., which components influence which others.}
Assume we extract $N$ components (which may be MLP neurons, attention heads, or SAE features) from all the desired locations in the LLMs and concatenate {their activations} to form a vector $\vx = [x_1, x_2, \cdots, x_N] \in \sR^N$. Our goal is to learn the DAG $\gG$ with the $N$ components as its nodes.
The structural equation models (SEMs) 
model the structural relations between a component $x_i$ and its parents $\parents_\gG(x_i)$: $x_i = f_i(\parents_\gG(x_i)) + \varepsilon_i$, where $f_i(\cdot)$ is the structural function and $\varepsilon_i$ is the residual. In this work, we assume the structural relations between components are linear. Given $M$ observations of the $N$ components, i.e., input matrix $\mX \in \sR^{N\times M}$, we can then obtain the linear SEM in its matrix form:
\begin{equation}\label{eq:linearSEM}
    \scalebox{1}{$
    \mX = A^\top \mX + \bm{\varepsilon},
    $}
\end{equation}
with continuous parameters $A\in \sR^{N\times N}$, a weighted adjacency matrix; $\bm{\varepsilon} \in \sR^{N\times M}$ is a matrix of linearization errors.
$A[i,j] \neq 0$ indicates the directed dependency $x_i \to x_j$. We aim to learn $A$ by minimizing the reconstruction error between $\mX$ and $A^\top \mX$ subject to sparsity and acyclicity constraints:
\begin{equation}\label{eq:circuit_prob}
    \scalebox{0.9}{$
    \begin{split}
        \widehat{A} =& \arg\min_{A} \|\mX - A^\top\mX\|_F^2 + \lambda \|A\|_1, \\
        & \text{subject to } \mathcal{G}(A) \in \sD
    \end{split} 
    $} 
\end{equation}
where $\|\cdot\|_F$ denotes Frobenius norm; $\|A\|_1$ is the sparsity penalty with tuning coefficient $\lambda$; $\mathcal{G}(A)$ is the circuit structure inferred from $A$; and $\sD$ is the space of acyclic graphs with $N$ nodes.

\textbf{Discussion of Assumptions.} We use Eq.~(\ref{eq:linearSEM}) as a sparse-regression \emph{surrogate} for the LLM's nonlinear computational graph, not as a strict SEM. Accordingly, we do not appeal to identifiability theorems that would require causal sufficiency, independent noise, or a correctly specified linear functional form: all three hold only approximately in transformer LLMs given the residual stream, SAE reconstruction error, and the inherent nonlinearity of attention, LayerNorm, and MLP layers. In particular, $\varepsilon$ should be interpreted as \emph{linearization error} rather than exogenous noise, since LLM activations are deterministic. Our goal is to recover the dependency \emph{skeleton} over modeled components efficiently. The dependency skeleton is a sparse summary of the linear-projection structure among the modeled components, with unmodeled contributions absorbed into the residual term. The $\ell_1$ penalty filters out weak dependencies regardless of whether they originate from omitted parents or from genuine but small effects, so the recovered skeleton remains a useful and conservative map of the strongest dependencies that exist among the components included in the model. We also provide a nonlinear extension of \mn{}, which empirically yields similar topological skeletons at higher computational cost. Please refer to Appendix \ref{app:assumptions} for a detailed discussion.

The main computational challenge of the optimization in Eq.~(\ref{eq:circuit_prob}), as in many continuous causal discovery approaches, lies in enforcing the acyclicity constraint. Such a constraint is essential to prevent self-loops and cycles, which are unsuitable for interpreting the transmission, aggregation, and evolution of model components. To address this, we make simplifying assumptions that bypass the explicit enforcement of this constraint and reduce the optimization to sparse linear regression problems (i.e., Lasso), enabling a scalable solution.

\subsection{Circuit Discovery on Neurons}\label{sec:cd_neuron}
To better understand how models encode and process information, mechanistic interpretability research~\citep{conmy2023towards,cao2021low,syed2024attribution} has focused on identifying graphical structures (circuits) connecting pre-trained language model neurons, including outputs from attention and MLP modules. To evaluate the effectiveness of our proposed method in Section~\ref{sec:general_formulation}, we follow the same setting as these prior works and treat model neurons as the components of interest. We first collect neuron activations $[\vh_1, \vh_2, \ldots, \vh_L]$ from $L$ target locations, each with dimension $d$, and for $M$ LLM inputs (observations), resulting in $\mH \in \sR^{L\times d\times M}$. Existing circuit discovery methods typically assume that circuit structures respect the model locations' computation order, meaning that neurons from layer $i$ precede those from layer $j$ if $i < j$, and within each layer, attention activations come before MLP activations. We adopt this assumption to simplify the acyclicity constraint in Eq.~(\ref{eq:circuit_prob}). Accordingly, we reorder $\mH$ according to the computational graph and reshape it to obtain $\Tilde{\mH}\in\sR^{N\times M}$, whereby $N=Ld$. Substitute $\mX$ in Eq.~(\ref{eq:circuit_prob}) with $\Tilde{\mH}$ to estimate the weighted adjacency matrix $A$ as
\begin{equation}\label{eq:neuron_obj}
    \begin{split}
        & \widehat{A} = \arg\min_{A} \|\Tilde{\mH} - A^\top \Tilde{\mH}\|_F^2 + \lambda \|A\|_1, \\
    & \text{subject to } A \text{ being block upper triangular.}
    \end{split}
\end{equation}
Specifically, each block $A[i,j]$ is now a $d \times d$ square matrix, and $A[i,j] = 0^{d \times d}$ whenever $i \geq j$, so that each block of variables depends only on preceding blocks and later layers cannot influence earlier ones. The constraints this imposes, and why they are justified for transformer architectures, are discussed together with the rest of the surrogate's assumptions in Appendix~\ref{app:assumptions}.
In practice, we enforce this constraint by initializing the lower-triangular blocks to zero matrices and keeping them fixed throughout optimization. The resulting circuit structure $\mathcal{G}$ is then inferred from $\widehat{A}$.

We now provide complexity analysis of our proposed circuit discovery approach on model neurons versus the existing intervention-based approaches. For our optimization problem in Eq.~(\ref{eq:neuron_obj}), we have:
\begin{proposition}\label{our_bigo_model_neuron}
    \textbf{Computational complexity of \mn{} on model neurons.}
    Consider the block upper-triangular Lasso problem in Eq.~(\ref{eq:neuron_obj}), with $\tilde{\mH} \in \sR^{N \times M}$ containing $M$ observations of $N = Ld$ neuron activations from $L$ positions of width $d$.  The objective is the sum of a smooth quadratic data-fit term and a non-smooth $\ell_1$ regularizer; on this composite objective, FISTA~\citep{beck2009fast} reaches $\epsilon$-suboptimality in $\mathcal{O}(1/\sqrt{\epsilon})$ iterations, matching the optimal rate of Nesterov's accelerated gradient method on smooth convex problems.  The total computational cost is therefore $\mathcal{O}\!\left(\frac{M\,L(L-1)\,d^{2}}{2\,\sqrt{\epsilon}}\right)$.
\end{proposition}
Please refer to the detailed proof in Appendix~\ref{proof_bigo_neuron}. We compare the cost of \mn{} against EAP-ig\footnote{We omit the theoretical and
empirical comparison to RelP~\citep{jafari2025relp}, as it inherits the runtime of EAP-ig and our
reported efficiency advantage over EAP-ig already characterizes the
expected gap against RelP.}~\citep{hanna2024have}, the state-of-the-art intervention-based circuit discovery method.  EAP-ig estimates indirect effects (IEs) over $L$ component locations of dimension $d$ via \emph{attribution patching}, which linearizes and parallelizes the computation so its additional bookkeeping scales only as $\mathcal{O}(MLd)$.  The dominant cost is therefore the two forward and one backward passes through the LLM that EAP-ig requires per observation.  \mn{}, by contrast, runs no backward passes and shares only the initial forward pass (used for activation collection) with EAP-ig.  We next quantify this gap on the same problem (Eq.~\ref{eq:neuron_obj}) and identify regimes in which \mn{} is provably faster.

\begin{proposition}\label{efficient_conditions}
    Consider a transformer-based large language model with $S$ blocks and model neuron dimension $d$. $n_{\text{token}}$ is the token sequence length, $h$ is the number of attention heads, and $f$ is the feedforward expansion factor. Beyond the shared cost, EAP-ig requires an additional cost of $\mathcal{O}\!\left(12MSn_{\text{token}}d^2(2+f) + 12MSn^2_{\text{token}}d + MLd\right)$, while \mn{} with computational complexity $\mathcal{O}\!\left(ML(L-1)d^2\,/\,2\sqrt{\epsilon}\right)$ has guaranteed better efficiency compared to EAP-ig if one of the following conditions holds: (1) $L(L-1)\,/\,24\,S(2+f)\sqrt{\epsilon} < n_{\text{token}} \ll d$;
    (2) $n_{\text{token}} \gg d$ and $n_{\text{token}} > \sqrt{L(L-1)\,d\,/\,24\,S\,\sqrt{\epsilon}}$. 
\end{proposition}

Please refer to the detailed proof in Appendix~\ref{proof_efficient_conditions}. Intuitively, \textbf{Proposition~\ref{efficient_conditions}} provides a guideline for selecting the number of component locations of interest ($L$) and determining which circuit discovery approach is more efficient.

\subsection{Circuit Discovery on Sparse Features}
Recent work~\citep{markssparse,laptev2025analyze} has shifted circuit discovery from raw neurons to monosemantic SAE features, which yield cleaner, more human-interpretable circuits. However, high SAE dimensions push the cost of intervention-based methods beyond practical limits. We address this by extending our sparse-regression formulation to operate directly on SAE features.
\subsubsection{Preliminaries on Sparse Autoencoders}\label{sec:sae}
Given a model with $d$-dimensional latent space and neuron activations $\vh \in \sR^{d}$, an SAE can represent $\vh$ as a linear combination of sparse features $\vz\in \sR^{D}, D \gg d$:
\[\vz = \sigma(\mW_{\text{enc}} \vh + \vb_{\text{enc}}), \quad \hat{\vh} = \mW_{\text{dec}} \bm z + \vb_{\text{dec}},\]
where $\mW_{\text{enc}}, \vb_{\text{enc}}$ are encoder parameters and $\mW_{\text{dec}}, \vb_{\text{dec}}$ are decoder parameters; $\sigma(\cdot)$ is a nonlinear activation function. The SAEs are usually trained by minimizing the reconstruction error between model activations $\vh$ and reconstructed activations $\hat{\vh}$ subject to a sparsity regularizer:
\[\vspace{-2 mm}\gL_{\text{SAE}}\coloneqq \|\vh - \hat{\vh}\|_2^2 + \alpha \gL_{\text{reg}}(\vz).\vspace{-0.5 mm}\]
Recent work on SAEs explores various methods to achieve a better trade-off between fidelity and sparsity.
\textbf{However, our focus is to uncover the structural dependencies among learned sparse SAE features, rather than to develop new SAE training methods.} We directly employ pre-trained SAEs on LLMs of various sizes. Details can be found in Appendix~\ref{app:sae}.

\subsubsection{Layer-wise Sparse Feature Circuit Discovery}\label{sec:cd_sae_features}
We follow our formulation on neurons and still assume that structural relations follow the computation order of the underlying model neurons. To be specific, consider two model neurons of dimension $d$ and their activations at locations $i$ and $j$, denoted by $\vh_i, \vh_j \in \sR^d$, where computation at $i$ precedes computation at $j$. We obtain the corresponding SAE features $\vz_i, \vz_j \in \sR^D$ using trained SAEs:
\begin{equation}
    \scalebox{0.75}{$\vz_i = \sigma (\widehat{\mW}_{\mathrm{enc},i} \vh_i + \widehat{\vb}_{\mathrm{enc},i}), \vz_j = \sigma (\widehat{\mW}_{\mathrm{enc},j} \vh_j + \widehat{\vb}_{\mathrm{enc},j}).$}
\end{equation}
If a dependency exists between variables in $\vz_i$ and $\vz_j$, we constrain its direction to be from $i$ to $j$. Given $M$ observations of $\vz_i$ and $\vz_j$, we obtain input data $\mZ_i \in \sR^{D\times M}$ and $\mZ_j\in\sR^{D\times M}$. We estimate these relations by solving:
\begin{equation}\label{eq:layer2layer_obj}
    \widehat{A}_{i,j} = \arg\min_{A_{i,j}} \|\mZ_j - A_{i,j}^\top \mZ_i\|_F^2 + \lambda \|A_{i,j}\|_1, \vspace{-3 mm}
\end{equation}
where $A_{i,j} \in \sR^{D\times D}$. This procedure is repeated for every pair $(i, j)$ where $i$ precedes $j$ in the computation order. \textbf{In particular, learning $A_{i,j}$ for all transformer block outputs in consecutive layers provides insight into how semantic concepts are transferred and propagated, and how they evolve, across the model.} The computational cost of the learning problem in Eq.~(\ref{eq:layer2layer_obj}) is $\gO(\frac{MD^2}{\sqrt{\epsilon}})$.

\begin{figure*}[hpt]
    \centering
    \includegraphics[width=1.0\linewidth]{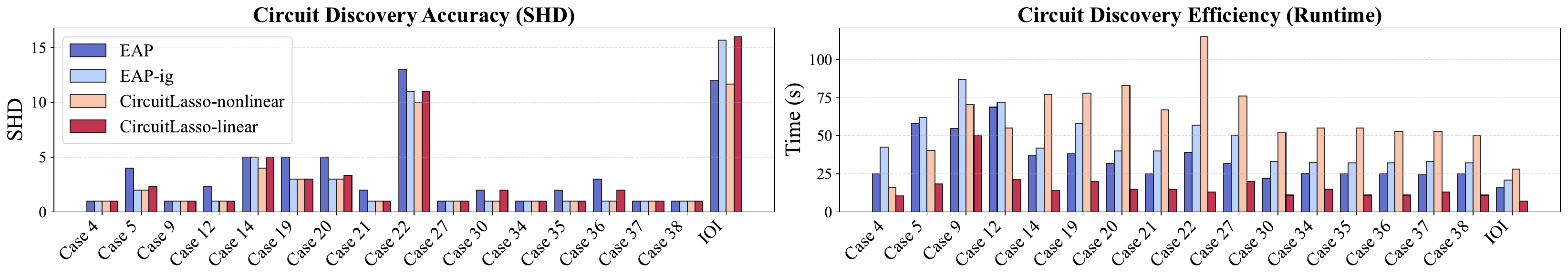}
    \caption{\textbf{Circuit discovery accuracy and efficiency on InterpBench.} We compare our \mn{} against the SOTA circuit learning baselines (EAP, EAP-ig) across 16 synthetic cases and the real IOI task. \textbf{Left:} SHD between recovered and ground-truth circuits (lower is better); \textbf{Right:} wall-clock runtime in seconds (lower is better).}
    \label{fig:circuit_discovery_model_neuron}
    \vspace{-5 mm}
\end{figure*}

We also incorporate the downstream prediction target into circuit discovery to enable explanation of the model's predictive behavior. We formulate the following optimization problem to learn a model for predicting the downstream target $y$ using SAE features $\vz_i$, derived from model neuron activations at location $i$:
\begin{equation}\label{eq:layer2y_obj}
    \widehat{A}_{i,y} = \arg\min_{A_{i,y}} \gL_{\text{pred}}(y, A^\top_{i,y}\mZ_i) + \lambda \|A_{i,y}\|_1,
\end{equation}
where $A_{i,y} \in \sR^{D}$ and $\gL_{\text{pred}}(\cdot, \cdot)$ denotes the prediction loss\footnote{Mean squared error for regression tasks and cross-entropy loss for classification tasks.}. In practice, with the learned $A_{i,y}$ and interpretable sparse features $\vz$, we can not only explain the model's predictive behavior, but also rectify the prediction model to mitigate spurious or biased behavior. The computational cost of Eq.~(\ref{eq:layer2y_obj}) is $\gO(\frac{MD}{\sqrt{\epsilon}})$.

\section{Experiments}

We evaluate \mn{} on both model neurons (Section~\ref{sec:cd_neuron_exp}) and sparse autoencoder features (Section~\ref{sec:cd_sf}).  Our central claim is \emph{efficiency at parity of accuracy}: \mn{} recovers circuits whose structural accuracy matches state-of-the-art intervention-based methods, at a fraction of the computational cost.  We then apply \mn{} to a setting whose inner workings have not previously been studied in mechanistic interpretability, and show that the recovered SAE-feature circuits yield interpretable insights.  Finally, we apply our learned circuits to a benchmark domain-generalization (DG) task, achieving comparable, and in some cases slightly better, performance at substantially lower cost.

\subsection{Circuit Discovery on Model Neurons}\label{sec:cd_neuron_exp}
\textbf{Setup.} We evaluate \mn{} on \textsc{InterpBench}~\citep{gupta2024interpbench}, a collection of 86 semi-synthetic transformers with known ground-truth circuits, against the state-of-the-art Edge Attribution Patching (EAP)~\citep{syed2024attribution} and EAP with integrated gradients (EAP-ig)~\citep{hanna2024have}. Following the protocol of \citet{gupta2024interpbench}, we evaluate on the \emph{16 main synthetic cases} and the \emph{real Indirect Object Identification (IOI) case}. Accuracy is measured by Structural Hamming Distance (SHD) and efficiency by runtime in seconds, averaged over three trials on a single NVIDIA A100. The full learning procedure is in Appendix~\ref{app:neurons}.

\noindent\textbf{Results} (Figure~\ref{fig:circuit_discovery_model_neuron}).
\emph{(i) Comparable accuracy at substantially lower cost.} Across the 17 tasks, \mn{}-linear attains a mean SHD of $3.16$, on par with EAP-ig ($2.98$) and below EAP ($3.61$), at a mean runtime of $16.3$\,s per case, $3.0$ times faster than EAP-ig ($49.1$\,s) and $2.1$ times faster than EAP ($33.7$\,s). In short, \mn{} recovers circuits as accurately as the strongest intervention-based baseline at a fraction of the cost, directly supporting our ``efficiency at parity of accuracy'' claim.
\emph{(ii) The nonlinear variant brings diminishing returns.} \mn{}-nonlinear achieves the lowest mean SHD ($2.84$), only marginally below \mn{}-linear, but at $3.7$ times the runtime, even slower than EAP-ig in most cases. This is consistent with the intuition that the extra nonlinear capacity is wasted when the dependency structure is already well-captured by a linear edge-importance score.
\emph{(iii) The comparison is descriptive, not a significance test.} Run-to-run variation on \textsc{InterpBench}'s fixed synthetic transformers is small: SHD has zero standard deviation for all three methods on 11 of the 16 cases. \mn{}-linear matches EAP-ig exactly, in mean and standard deviation, on 11 cases and differs by at most $1$ SHD on the remaining $5$ (Table~\ref{tab:shd_percase}); we claim comparability, not statistical equivalence.
\emph{(iv) Errors are mostly missed edges.} According to Table~\ref{tab:shd_errors}, false negatives account for $81\%$ of \mn{}-linear's errors, $75\%$ of EAP-ig's and $72\%$ of EAP's, and \mn{}-linear has the lowest false-positive rate ($0.44$ versus $0.54$ and $0.85$), at a slightly higher false-negative rate than EAP-ig ($1.90$ versus $1.65$).

\subsection{Circuit Discovery on SAE Features}\label{sec:cd_sf}
We adapt \mn{} to learn structural circuits on SAE features (as described in Section~\ref{sec:cd_sae_features}), revealing model behaviors in terms of human-interpretable concepts (Section~\ref{sec:interp}). Following~\citet{markssparse}, we also leverage insights from learned circuits for downstream domain generalization (Section~\ref{sec:dg}).

\subsubsection{Interpretability Case Studies}\label{sec:interp}
\textbf{Data and Model.} We demonstrate our approach on the \textbf{Co}rpus of \textbf{L}inguistic \textbf{A}cceptability (CoLA) task~\citep{warstadt2018neural} from the \textsc{glue} benchmark~\citep{wang2018glue}, aiming to reveal the inner workings of GPT-2 small~\citep{radford2019language} through interpretable features derived from OpenAI's pre-trained sparse autoencoders~\citep{gao2024scaling} for GPT-2 small. We present this section as a \emph{qualitative} case study of feature-level circuit interpretability on GPT-2 small\footnote{The observations below characterize this model and SAE dictionary, and we do not claim that they transfer unchanged to larger or more modern architectures.}. We chose CoLA for its linguistic-acceptability structure and the abundance of monosemantic SAE features it elicits, rather than as a dataset contribution. It contains 10,657 sentences from 23 linguistics publications, annotated for grammaticality by the original authors. We conduct our interpretability experiments on the 8,551 training sentences in the public release.

\begin{figure*}[hpt]
    \centering
    \includegraphics[width=0.85\linewidth]{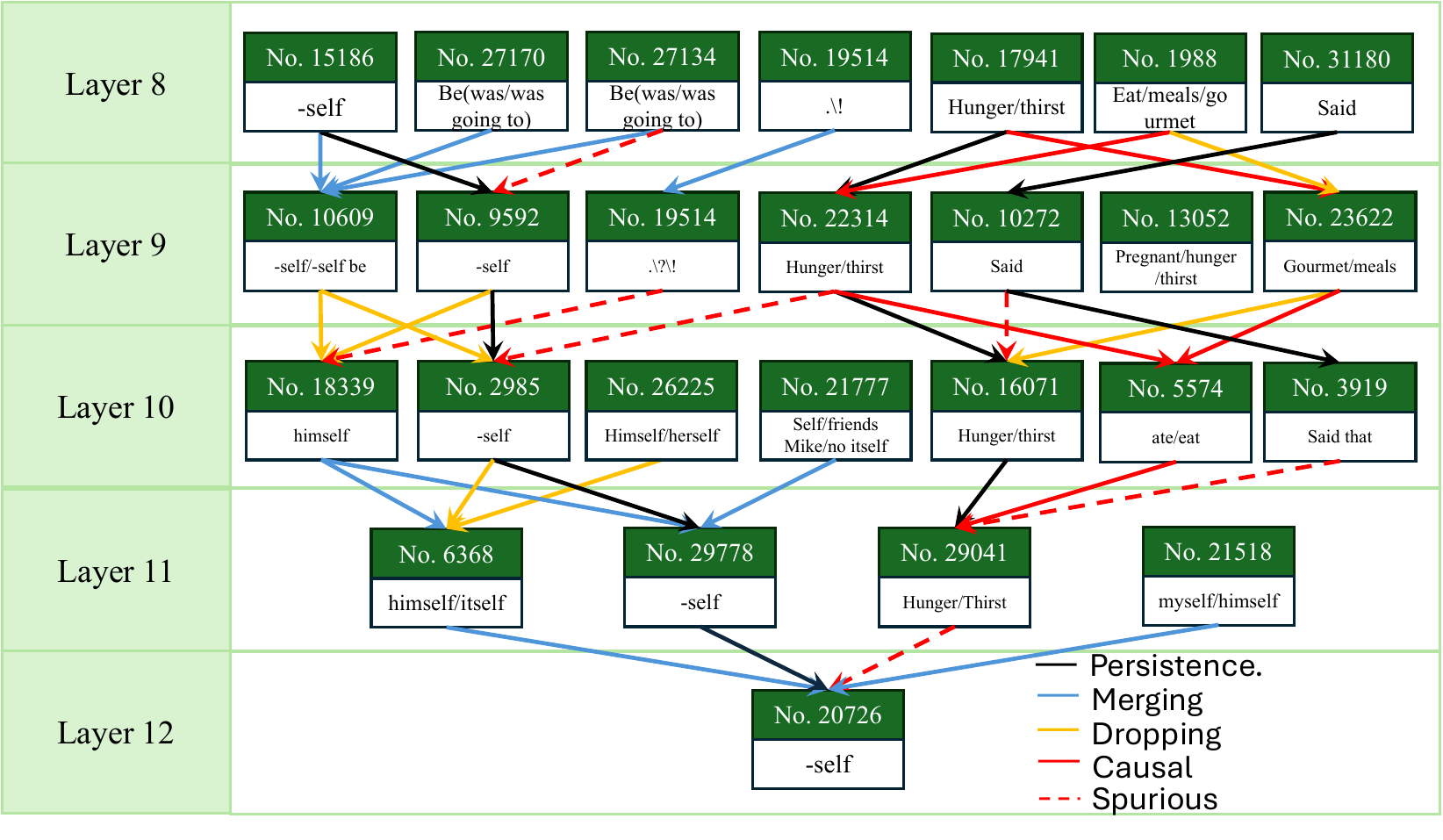}
    \caption{\small{The learned circuits over SAE features on GPT-2 small model. Different colors represent different types of edges in the computation graph.  The figure illustrates one example of a valid circuit path traced through the learned adjacency matrices: starting from four interpretable final-layer features ($z_{12,20726}$, $z_{12,776}$, $z_{12,19322}$, $z_{12,3092}$ labeled ``-self,'' ``hunger/thirst,'' ``tired/weary,'' and ``ending punctuation''), we recurse backward through $A_{i-1,i}$ to surface their most influential parents at each preceding layer.  We do not impose a fixed cap on the number of features per layer or a global coefficient threshold; rather, the figure shows a representative traversal whose roots and parents are chosen for human-interpretable illustration, and an additional traversal starting from a different root is reported in Appendix~\ref{app:sparse_features_circuit}.}}
    \label{fig:example_circuit_1}
\end{figure*}
\textbf{Sparse Feature Interpretation within Learned Circuits.}
We extract GPT-2 small's neuron activations and the corresponding SAE features on the $M$ training sentences (Appendix~\ref{app:sae_features}).

We construct circuits at two granularities, controlled by how the importance score is computed.  For the dataset-level interpretability case study reported below (Figure~\ref{fig:example_circuit_1} and Appendix~\ref{app:sparse_features_circuit}), we rank features at the prediction layer by the \emph{dataset-level coefficient} $|A_{L,y}|$ obtained from the regression fit jointly on all $M$ training prompts, yielding a single circuit that summarizes the model's behavior across the dataset.  When the goal is to explain a single decision, we instead use the prompt-specific Hadamard product $\vs = |A_{L,y}| \odot |\vz_L| \in \sR^D$ for the prompt of interest, which preserves the learned linear coefficients but reweights them by that prompt's activation pattern.  Prompt averaging is appropriate when the target object of study is a population-level dependency, e.g., the mechanism by which the model generally distinguishes grammatical from ungrammatical sentences, while the prompt-specific score is appropriate when the target is per-instance mechanism.  All circuits shown in the main paper use prompt averaging; per-prompt analyses are illustrated in Appendix~\ref{app:sae_features}.

For each important feature, we label its semantic concept using two complementary procedures: a \emph{multi-prompt} approach that aggregates tokens which strongly activate the feature across many prompts, and a \emph{single-prompt} approach that perturbs targeted tokens in a single prompt and verifies that the activation changes accordingly.
We then trace each important feature backward through earlier layers via the learned adjacency $A_{i-1,i}$ and label its parents with the same procedures, yielding tree-shaped circuit paths across transformer blocks.
Full procedure, worked examples (e.g., feature $z_{12,20726}$ for ``-self,'' $z_{12,776}$ for ``thirst/hunger''), and per-feature tables are deferred to Tables \ref{tab:semantic_multi_20726}, \ref{tab:semantic_multi_776}, \ref{tab:semantic_multi_19322}, and \ref{tab:semantic_multi_3092} in the same section of Appendix.

Figure~\ref{fig:example_circuit_1} presents such a tree-shaped circuit, consisting of sparse features with human-interpretable meanings across 5 layers. More examples are provided in Appendix~\ref{app:sparse_features_circuit}. From the circuits in Figure~\ref{fig:example_circuit_1}, we make the following observations:

\textbf{Persistence.} Certain semantic concepts persist along circuit paths across multiple layers, particularly in the later layers. For example, the concept of ``-self'' is present in the $20726^{\text{th}}$ feature of layer 12, the $6368^{\text{th}}$ feature of layer 11, the $2985^{\text{th}}$ feature of layer 10, the $9592^{\text{th}}$ feature of layer 9, and the $15186^{\text{th}}$ feature of layer 8. We highlight circuit paths that capture persistence relations between consecutive layer features in black.

\textbf{Merging and Dropping.} We also observe that sparse features in later layers can merge semantic concepts from multiple parent features in the preceding layer, or disregard (i.e., drop) certain concepts contributed by those parent features. For instance, the $10609^{\text{th}}$ feature of layer $9$ merges concepts from both the $15186^{\text{th}}$ and $27170^{\text{th}}$ features of layer $8$. In contrast, the $6368^{\text{th}}$ feature of layer 11 retains only the concept ``himself'' and disregards all other forms of ``-self'' from the $2985^{\text{th}}$ feature of layer $10$. We highlight circuit paths representing propagation in blue and decomposition in orange.

\textbf{Cause-Effect and Spurious Correlations.} Our circuits can capture causal relations between features that encode cause--effect semantic concepts. For example, the $22314^{\text{th}}$ feature of layer $9$ represents the concept of ``hunger/thirst,'' which can be considered a cause of the action ``ate/eat,'' encoded in the $5574^{\text{th}}$ feature of layer $10$. However, our assumption that causal orientations align with the computation order results in some circuit paths appearing anti-causal. From a human perspective, one typically feels hungry before taking actions such as ``eat food/meals/gourmet,'' yet our circuit includes a path from the $1988^{\text{th}}$ feature of layer $8$ to the $22314^{\text{th}}$ feature of layer $9$, which implies the reverse. Moreover, our circuits also capture spurious correlations. For example, the $20726^{\text{th}}$ feature of layer $12$, which represents ``-self,'' is spuriously correlated with the ``hunger/thirst'' concept encoded in the $29041^{\text{th}}$ feature of layer $11$. Such correlations are likely introduced by biases in the training data, such as the frequent co-occurrence of these two semantic concepts within the same sentence. By analyzing these circuit paths, we can infer the nature of dataset biases and potentially mitigate them through targeted model editing. We next show how such insights from a learned circuit can be leveraged to improve downstream domain generalization in Section~\ref{sec:dg}.

\textbf{Faithfulness and Completeness.}
We evaluate the learned SAE-feature circuits on CoLA using the faithfulness and completeness metrics of \citet{markssparse}.  As Figure~\ref{fig:fc} shows, our circuits match the intervention-based \textsc{shift} baseline under standard node ablation, while requiring no per-edge interventions.  Because our sparse-regression framework learns explicit edge coefficients, we additionally introduce \emph{edge ablation}, a capability that \textsc{shift} does not support, and confirm that a small subset of essential edges (together with their connected SAE features) governs the model's prediction behavior.  Full methodology and per-quartile results are in Appendix~\ref{app:faithfulness}.
\begin{figure}[hpt]
    \centering
    \includegraphics[width=\columnwidth]{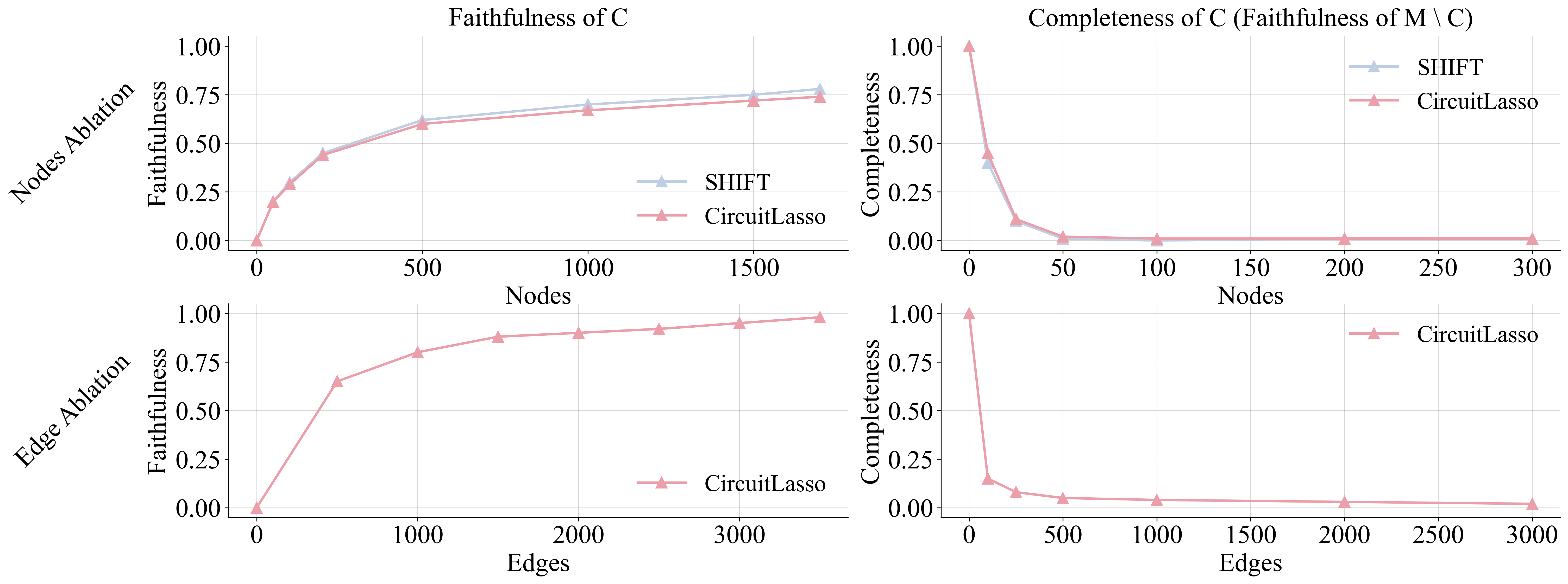}
    \caption{Faithfulness and completeness scores for the learned circuits on the CoLA training set, under node ablation (top) and edge ablation (bottom). Ideal faithfulness is 1, ideal completeness is 0.}
    \label{fig:fc}
\end{figure}

\subsubsection{Downstream Domain Generalization: A Utility Demonstration}\label{sec:dg}
\begin{table*}[hpt]
\centering
\small
\caption{Runtime and numbers of selected features for \textsc{shift} versus our \mn{} method. 
The runtime does not include manual interpretation time.}
    \label{tab:bib_pythia_gemma_eff}
\scalebox{1}{
\renewcommand{\arraystretch}{1.2}
\begin{tabular}{l|cc|cc|cc}
\toprule
\textbf{Method} 
& \multicolumn{2}{c|}{\textbf{Pythia-70M}} 
& \multicolumn{2}{c|}{\textbf{Gemma-2-2b}} 
& \multicolumn{2}{c}{\textbf{Gemma-2-9b}}\\
& \scriptsize{\# of features}
& \scriptsize{Runtime (s)} $\downarrow$
& \scriptsize{\# of features}
& \scriptsize{Runtime (s)} $\downarrow$
& \scriptsize{\# of features}
& \scriptsize{Runtime (s)} $\downarrow$
\\
\midrule
\textsc{shift} & \multirow{2}{*}{49} & 257.6 & \multirow{2}{*}{65} & 371.2 & \multirow{2}{*}{71} & 908.4 \\

\textsc{shift}-retrain & & 356.3 & & 476.8 & & 1056.0 \\
\midrule
\scriptsize{\mn{}} & \multirow{2}{*}{41} & 36.5 & \multirow{2}{*}{55} & 47.2 & \multirow{2}{*}{59} & 107.4 \\

\scriptsize{\mn{}}-retrain & & 61.9 (17.37\%\textcolor{red}{$\uparrow$}) & & 72.5 (15.20\%\textcolor{red}{$\uparrow$}) & & 125.2 (11.98\% \textcolor{red}{$\uparrow$}) \\
\bottomrule
\end{tabular}
}
\end{table*}

\begin{table*}[h!]
\centering
\caption{Prediction accuracy (\%) on the Bias-in-Bios dataset across three LLMs and debiasing methods.  For \emph{Prof.}~(profession-prediction accuracy on the balanced test set) and \emph{Worst}~(accuracy on the worst-performing demographic subgroup), higher is better.  For \emph{Gender}~(how predictable gender is from the predictor's representation), values close to 50\% are better, indicating that the predictor does not rely on the spurious gender signal; values much above 50\% indicate residual gender leakage, while values much below 50\% indicate over-correction.  \textsc{oracle} is trained on the balanced set and shown as a non-comparable upper bound; bold marks the best non-\textsc{oracle} entry per column.}
    \label{tab:bib_pythia_gemma}
\small
\scalebox{0.9}{
\begin{tabular}{lccccccccc}
\toprule
\textbf{Method}
& \multicolumn{3}{c}{\textbf{Pythia-70M}} 
& \multicolumn{3}{c}{\textbf{Gemma-2-2B}} & 
\multicolumn{3}{c}{\textbf{Gemma-2-9B}} \\
\cmidrule(lr){2-4} \cmidrule(lr){5-7} \cmidrule(lr){8-10}
& \scriptsize{Profession ($\uparrow$)}
& \scriptsize{Gender}
& \scriptsize{Worst group($\uparrow$)}
& \scriptsize{Profession($\uparrow$)}
& \scriptsize{Gender}
& \scriptsize{Worst group($\uparrow$)}
& \scriptsize{Profession($\uparrow$)}
& \scriptsize{Gender}
& \scriptsize{Worst group($\uparrow$)} 
\\
\midrule
\textsc{original} & 61.9 & 87.4 & 24.4 & 69.6 & 79.5 & 4.1 & 70.8 & 78.2 & 23.4\\
\textsc{cbp} & 83.3 & 60.1 & 67.7 & 90.1 & \textbf{50.2} & 86.8& 94.7 & \bf{50.0} & 91.5 \\
\textsc{oracle} & 93.0 & 49.4 & 91.4 & \textbf{95.1} & \bf{50.2} & 91.7 & 95.7 & \bf{50.0} & 90.5 \\
\textsc{linear probing}            & 84.0          & 58.0          & 70.0          & 90.5          & 51.5          & 87.5          & 95.0          & 50.5          & 91.0          \\
\textsc{shift}& 88.5 & 54.0 & 76.0 & 72.8 & 51.6 & 43.7 & 77.1 & 52.8 & 67.9  \\
\textsc{shift}-retrain& 93.1 & 52.0 & \textbf{89.0} & 94.2 & 52.4 & 92.4 & 96.0 & 51.3 & 90.3 \\
\midrule
\scriptsize{\mn{}} & 90.5 & \bf{50.1} & 75.8 & 77.5 & 50.7 & 50.5 & 81.5 & 50.3 & 69.8 \\
\scriptsize{\mn{}}-retrain& \textbf{94.2} & 50.6 & 88.7  & \textbf{95.1} & 52.8 & \textbf{92.9} & \bf{96.9} & 50.5 & \bf{91.5} \\
\bottomrule
\end{tabular}
}
\end{table*}

We treat this section as a \emph{utility demonstration} of the insights from our learned circuits, not as a multi-layer circuit learning experiment.  Our central claim remains efficiency improvements and any accuracy advantage is interpreted as a side benefit of operating on disentangled SAE features, not as a contribution per se.

\textbf{Setup.}
We evaluate on the Bias-in-Bios dataset (BiB)~\citep{de2019bias} on Pythia-70M~\citep{biderman2023pythia}, Gemma-2-2B and Gemma-2-9B~\citep{gemma_2024}, using the pre-trained SAEs from~\citet{markssparse} for Pythia-70M and \citet{lieberum2024gemma} for Gemma-2-2B and Gemma-2-9B.  The task is to predict profession from a biography under a gender-spurious training set and evaluate on a balanced test set in which profession and gender are independent.  We compare against the following baselines. \textsc{original} is a predictor trained directly on the gender-spurious set, and \textsc{oracle} is a predictor trained on the balanced set, included as a non-comparable upper bound. \textsc{cbp} (concept bottleneck probing)~\citep{yan2023robust} is a supervised feature-filtering approach that routes the prediction through a bottleneck layer. \textsc{shift}~\citep{markssparse} is the intervention-based counterpart to our method: it identifies causally relevant SAE features by activation patching and ablates the spurious ones. \textsc{linear probing}~\citep{gurnee2023finding} is $k$-sparse probing on SAE features. Throughout, the ``-retrain'' suffix (as in \textsc{shift}-retrain and \mn{}-retrain) denotes retraining the downstream linear classifier on the ablated features; it does not alter the circuit-discovery step itself.

\textbf{Method.}
We rank SAE features at a single layer by $|\widehat{A}_{i,y}|$, manually identify gender-correlated features, zero them, and feed the result either directly into the trained classifier (\mn{}) or into a freshly retrained classifier (\mn{}-retrain).  Full setup, gender-leakage, and worst-group columns are in Appendix~\ref{app:bib_full}.

\textbf{Results.}
The central efficiency claim is in Table~\ref{tab:bib_pythia_gemma_eff}: \mn{} matches the strongest baseline at substantially lower cost, with the gap widening as model size grows.  Table~\ref{tab:bib_pythia_gemma} then reports profession accuracy across the three LLMs.  \mn{} and \mn{}-retrain achieve performance comparable to, and in some cases slightly better than, the strongest non-\textsc{oracle} baseline; we attribute the small accuracy advantage to finer-grained manipulation of disentangled SAE features rather than to a circuit learning effect.

\section{Conclusion and Discussion}
In this work, we presented a new circuit learning method, \mn{}, based on Lasso regression.  Our work offers a novel and effective solution to the challenges of polysemanticity and high dimensionality in LLM circuit learning.  By replacing computationally intensive intervention-based approaches with a sparse-linear-regression-based method, \mn{} achieves \emph{efficiency at parity of accuracy}, enabling scaling to high-dimensional SAE feature spaces.  Empirical studies show that the circuits discovered on SAE features are interpretable and useful for downstream tasks.  Our method's ability to handle high-dimensional data, its reliance on observational data alone, and the sparsity of its learned circuits represent significant advantages over existing baselines.

\section*{Limitations}
This work has limitations that we hope future research can address.  First, our linear structural-equation formulation is a deliberate first-order surrogate; while it recovers the topological skeleton accurately and efficiently, edge weights should not be over-interpreted as exact causal effects of the underlying nonlinear computation.  Second, the acyclicity constraint we use is supplied by the LLM's feed-forward computation order, a faithful constraint for transformer architectures but one that does not consider within-layer feedback, which may be relevant for nuanced architectures.  Third, the residual stream and SAE reconstruction error mean that causal sufficiency \emph{only} holds approximately.  Fourth, our qualitative feature-level circuit analysis is conducted on GPT-2 small and its SAE dictionary.  The labelling procedure is not model-specific and is also applied to Pythia-70M and Gemma-2 in Section~\ref{sec:dg}, but a full traced circuit on a modern model family is left to future work.  The propagation patterns reported in Section~\ref{sec:interp} should therefore be read as characterising GPT-2 small; whether they hold where SAE features are more entangled remains open.

\section*{Acknowledgments}
N.Y. and Y.Y. acknowledge support by the National Institutes of Health under award 1R01GM157589-01, the AFOSR Grant No. FA9550-22-1-0197, and the Defense Advanced Research Projects Agency (DARPA) under award HR00112590032. This research was, in part, funded by the U.S. Government by an agreement with Cornell University. N.Y. is also supported by IBM through the IBM–Rensselaer Future of Computing Research Collaboration.
Generative AI assistance was used solely for grammar and copy-editing of the
manuscript text.  All research contributions, experimental design, analyses,
and claims are the authors' own.

\bibliography{Reference}

\appendix
\section{Theoretical Proofs}

\subsection{Proof of Proposition~\ref{our_bigo_model_neuron}}\label{proof_bigo_neuron}
\begin{proof}
The block upper-triangular structure of $A$ decouples Eq.~(\ref{eq:neuron_obj}) into $L$ independent multi-output Lasso subproblems, one per output neural activations.  Specifically, the $\ell$-th subproblem predicts the $d$-dimensional $\ell$-th neuron activations from the $p_\ell = (\ell-1)\,d$ coordinates of all preceding (earlier-indexed) layer blocks; in particular $p_1 = 0$, since the first layer has no preceding blocks.

\emph{Per-iteration cost.}  Let $\vh_\ell \in \sR^{d \times M}$ denote the $\ell$-th neuron activation block and $\vh_{<\ell} \in \sR^{p_\ell \times M}$ denote the stacked activations of all preceding blocks.  Each FISTA iteration on the $\ell$-th subproblem evaluates the gradient of the quadratic data-fit term, dominated by the matrix product $\vh_{<\ell}\,\vh_{<\ell}^{\top}\,A_\ell$ with $A_\ell \in \sR^{p_\ell \times d}$, at cost $\mathcal{O}(p_\ell\,d\,M)$.  The subsequent soft-thresholding (proximal) step is $\mathcal{O}(p_\ell\,d)$ and is dominated.  Summing across the $L$ subproblems,
\begin{equation*}
    \sum_{\ell=1}^{L} p_\ell\,d \;=\; d^{2} \sum_{\ell=1}^{L} (\ell - 1) \;=\; \frac{d^{2}\,L(L-1)}{2},
\end{equation*}
so a single outer iteration across all $L$ subproblems costs $\mathcal{O}\!\left(\dfrac{M\,L(L-1)\,d^{2}}{2}\right)$.

\emph{Iteration complexity.}  The objective of each subproblem is the sum of a smooth quadratic with Lipschitz-continuous gradient (with constant $L_f^{(\ell)} = \lambda_{\max}\!\left(\vh_{<\ell}\,\vh_{<\ell}^{\top}\right)$) and a non-smooth $\ell_1$ regularizer.  FISTA on this composite objective satisfies the accelerated rate~\citep{beck2009fast}
\begin{equation*}
    F(A_\ell^{(k)}) - F(\widehat{A}_\ell) \;\le\; \frac{2\, L_f^{(\ell)} \,\|A_\ell^{(0)} - \widehat{A}_\ell\|_F^{2}}{(k+1)^{2}},
\end{equation*}
matching the optimal rate of Nesterov's accelerated gradient method on smooth convex problems.  Reaching $\epsilon$-suboptimality therefore requires only $\mathcal{O}(1/\sqrt{\epsilon})$ outer iterations.

Combining the per-iteration and iteration-complexity bounds gives the total cost $\mathcal{O}\!\left(\dfrac{M\,L(L-1)\,d^{2}}{2\,\sqrt{\epsilon}}\right).$
\end{proof}

\subsection{Proof for Proposition~\ref{efficient_conditions}}\label{proof_efficient_conditions}

\begin{proof}
Let $d$ be the model hidden dimension, $n_{\text{token}}$ the sequence length, $h$ the number of attention heads (so the head dimension is $k = d/h$), and $f$ the feedforward expansion factor.  Per transformer block, the dominant FLOP costs come from the Q/K/V projections ($\approx 6\,n_{\text{token}}\,d^2$), the $QK^\top$ products across all heads ($\approx 2\,n_{\text{token}}^2\,d$), the attention-weighted $V$ across heads ($\approx 2\,n_{\text{token}}^2\,d$), the output projection ($\approx 2\,n_{\text{token}}\,d^2$), and the two-layer feed-forward network of widths $d \to fd \to d$ ($\approx 4f\,n_{\text{token}}\,d^2$).  Summing, one forward pass per block costs $\mathcal{O}\!\left(n_{\text{token}}\,d^2(8+4f) + 4\,n_{\text{token}}^2\,d\right)$.  A backward pass costs approximately twice the corresponding forward pass, $\mathcal{O}\!\left(n_{\text{token}}\,d^2(16+8f) + 8\,n_{\text{token}}^2\,d\right)$.

\emph{Shared activation-collection cost.}  Both \mn{} and EAP-ig must perform one forward pass per observation in order to obtain the per-location activations on which their respective methods operate.  For $M$ observations across $S$ blocks, this shared cost is $\mathcal{O}\!\left(4MS\,n_{\text{token}}\,d^2(2+f) + 4MS\,n_{\text{token}}^2\,d\right)$, and we subtract it from both sides of the comparison below.

\emph{EAP-ig additional cost.}  Beyond activation collection, EAP-ig requires one further forward pass and one backward pass per observation in order to compute the attribution-patching gradients.  Across $S$ blocks and $M$ observations this is $\mathcal{O}\!\left(12MS\,n_{\text{token}}\,d^2(2+f) + 12MS\,n_{\text{token}}^2\,d\right)$.  The subsequent linear attribution-patching step over $L$ component locations contributes an additional $\mathcal{O}(MLd)$, which is dominated by the forward/backward terms.  EAP-ig's total additional cost is therefore
\[
\scalebox{0.9}{$\mathcal{O}\!\left(12MS\,n_{\text{token}}\,d^2(2+f) + 12MS\,n_{\text{token}}^2\,d + MLd\right).$}
\]

\emph{\mn{} additional cost.}  Beyond activation collection, \mn{} runs FISTA on the block upper-triangular Lasso problem of Eq.~(\ref{eq:neuron_obj}), with total cost $\mathcal{O}\!\left(\dfrac{ML(L-1)\,d^2}{2\sqrt{\epsilon}}\right)$ by Proposition~\ref{our_bigo_model_neuron}.

\emph{Efficiency conditions.}  For \mn{} to have strictly lower additional cost than EAP-ig, the dominant EAP-ig term must exceed the FISTA cost.  We consider the two regimes.
\begin{itemize}
    \item \emph{Case 1: $d \gg n_{\text{token}}$.}  The dominant EAP-ig term is $12MS\,n_{\text{token}}\,d^2(2+f)$.  Requiring this to exceed the FISTA cost gives
    \begin{align*}
        \begin{split}
            12MS\,n_{\text{token}}\,d^2(2+f) \;>&\; \frac{ML(L-1)\,d^2}{2\sqrt{\epsilon}} \\
            \Longrightarrow\;
    n_{\text{token}} \;>&\; \frac{L(L-1)}{24\,S(2+f)\sqrt{\epsilon}}.\\
        \end{split}
    \end{align*}
    \item \emph{Case 2: $n_{\text{token}} \gg d$.}  The dominant EAP-ig term is $12MS\,n_{\text{token}}^2\,d$.  Requiring this to exceed the FISTA cost gives
    \begin{align*}
        \begin{split}
            12MS\,n_{\text{token}}^2\,d \;>&\; \frac{ML(L-1)\,d^2}{2\sqrt{\epsilon}} \\
            \Longrightarrow\;
    n_{\text{token}}^2 \;>& \; \frac{L(L-1)\,d}{24\,S\,\sqrt{\epsilon}},\\
    \Longrightarrow\;n_{\text{token}}\;>&\;\sqrt{\frac{L\,(L-1)\,d}{24\,S\,\sqrt{\epsilon}}}.
        \end{split}
    \end{align*}
\end{itemize}

Combining the two regimes gives the stated conditions.
\end{proof}

\section{Discussion of Assumptions}\label{app:assumptions}
Treating Eq.~(\ref{eq:linearSEM}) as a strict linear structural equation model (SEM) would, following the continuous causal discovery literature~\citep{peters2014causal,zheng2018dags,park2020identifiability}, require causal sufficiency, a noise model that is either non-Gaussian or Gaussian with equal variance, and a correctly specified linear functional form. All three of these are at best approximations in transformer LLMs, and we explicitly do not appeal to identifiability theorems that depend on them. Instead, we use Eq.~(\ref{eq:linearSEM}) as a sparse regression surrogate whose purpose is to recover the dependency skeleton over the modeled components, which features influence which others, at a fraction of the cost of intervention-based circuit discovery. We outline below how each assumption is violated in our setting and why the skeleton-level conclusions of our method are nonetheless meaningful.

\paragraph{On Causal Sufficiency.} The residual stream of a transformer carries every preceding layer's contribution forward in parallel with the attention and MLP outputs we explicitly model, and the SAE reconstruction error introduces additional latent variation that is not captured by the recovered sparse features. Causal sufficiency therefore holds only approximately: activations at one location are not generated solely from the parents we include in $A$. We do not claim to recover a complete causal graph of the LLM. Rather, the dependency skeleton estimated by Eq.~(\ref{eq:circuit_prob}) is a sparse summary of the linear-projection structure among the modeled components, with unmodeled contributions absorbed into the residual term. The $\ell_1$ penalty filters out weak dependencies regardless of whether they originate from omitted parents or from genuine but small effects, so the recovered skeleton remains a useful and conservative map of the components included in the model.

\paragraph{On the Noise Model.} LLM activations are deterministic functions of the input prompt: the same prompt, evaluated twice, produces identical activations. The residual $\varepsilon$ in Eq.~(\ref{eq:linearSEM}) is therefore best understood as \emph{linearization error}, the deviation between the LLM's true nonlinear computation at each modeled component and the best linear projection of that component onto its predecessors over the observed prompts, rather than exogenous stochastic noise. In particular, we do not assume that $\varepsilon$ is independent of the parents, and we do not invoke noise distribution conditions (non-Gaussianity, equal variance) that the deterministic generative process would not satisfy in any case. The linear regression in Eq.~(\ref{eq:circuit_prob}) is justified as a \emph{projection} of the activation matrix onto the sparse linear hypothesis class, not as inference of an underlying stochastic SEM.

\paragraph{Linearity vs. Nonlinearity.} The transformer's true computational graph is globally nonlinear: attention, LayerNorm, and the activation functions inside MLP and SAE encoders all but guarantee that exact functional relationships between layer activations are not linear. We adopt the linear formulation as a deliberate first-order surrogate. Two observations justify its use for skeleton recovery. First, where a true nonzero dependency exists in the nonlinear computation, the corresponding linear regression coefficient is generically nonzero, Lasso under moderate model misspecification preserves support recovery up to constants, and edge \emph{presence} (not exact edge weight) is what we use downstream. Second, we introduce the nonlinear (kernel/MLP) extension of \mn{}: The linear formulation can be naturally extended to model nonlinear dependencies by representing the relationships between a variable and its parents through kernel functions or multi-layer perceptrons. The recovered topological skeleton is essentially the same as in the linear case, at substantially higher computational cost. We therefore treat the linear surrogate as the operational definition of the method. One can always employ the nonlinear formulation when finer-grained edge attribution is needed. Edge weights in $\widehat{A}$ should not be interpreted as exact causal effects of the underlying nonlinear computation; their role is to rank dependencies for selection, not to quantify them.

\paragraph{On Acyclicity.} The block upper-triangular mask introduced in Section~\ref{sec:cd_neuron} is the one structural constraint we impose, and it exploits the known, human-engineered computation order of the LLM rather than anything estimated from data. It aligns with the feed-forward nature of the network: activations in later layers are computed after, and depend on, those in earlier layers, so preserving that order requires no explicit acyclicity penalty. Supplying the ordering architecturally rather than learning it is what makes the recovered structure well-posed, and it is reflected in the empirical results of Figure~\ref{fig:circuit_discovery_model_neuron}. \citet{reisach2021beware} showed that continuous formulation of causal discovery methods, such as NOTEARS~\citep{zheng2018dags}, can succeed on synthetic benchmarks for the wrong reason, by exploiting marginal variances to recover the topological ordering, and that simply sorting variables by variance often matches their structural recovery. This critique does \emph{not} apply to our setting because the topological ordering is supplied externally, by the LLM's feed-forward computation order, rather than learned from the data. Concretely, our acyclicity constraint is implemented as a fixed block triangular mask on $A$ that respects the architectural layer order and the within-layer attention-then-MLP order; the optimizer never has access to a degree of freedom that could be co-opted by variance-based reordering. The Lasso step in Eq.~(\ref{eq:neuron_obj}) only selects which entries within the legal triangular pattern are nonzero. We therefore inherit none of the NOTEARS pathologies that the Reisach critique targets, and the acyclicity of the recovered circuit is a property of the architecture, not of an identifiability argument.  

\section{Experiment Details}\label{app:details}

\subsection{\mn{} Optimization with FISTA}\label{app:fista}
All instances of the \mn{} optimization in Eqs.~(\ref{eq:neuron_obj}), (\ref{eq:layer2layer_obj}), and (\ref{eq:layer2y_obj}) decompose into multi-output Lasso subproblems of the form $\min_A \tfrac{1}{2}\|\mX - A^\top \mX\|_F^2 + \lambda\|A\|_1$. We solve each subproblem with FISTA~\citep{beck2009fast} with backtracking line search, warm-started across both the neuron activations dimension and the regularization path. Predictor columns are standardized to unit $\ell_2$ norm before solving, and coefficients are rescaled afterward. We terminate when the relative objective decrease falls below $10^{-6}$. Our solver is implemented in PyTorch on a single NVIDIA A100 GPU. 

\subsection{\mn{} on Neurons} \label{app:neurons}
We begin by collecting activations from the model neurons. To ensure fairness, we use the same set of input prompts as the baselines. For example, in each case, EAP and EAP-ig employ two sets of data inputs: a clean run with $M$ input prompts, each with $n_{\text{token}}$ tokens, i.e., an $M \times n_{\text{token}}$ array of tokens $\mT_{\text{clean}}$, and a corrupted run of the same dimensionality, an $M \times n_{\text{token}}$ array $\mT_{\text{corrupted}}$. Our approach combines these two runs into a single dataset, an array $\mT = (\mT_{\text{clean}}, \mT_{\text{corrupted}})$ of $2M \times n_{\text{token}}$ tokens, which is then used to generate neuron activations at various locations in the LLM. Given $\mT$ and a pre-trained model, we obtain neuron activations at a location $i$ with shape $d\times 2M \times n_{\text{token}}$ and average over tokens to produce activations $\mH_i \in \sR^{d\times 2M}$. Repeating this process across all $L$ locations and sorting them according to the computation order yields $\tilde{\mH} \in \sR^{Ld \times 2M}$. Substituting $\tilde{\mH}$, the collected data matrix with $N = Ld$ dimensions and $2M$ observations into Eq.~(\ref{eq:neuron_obj}),
we aim to learn a weighted adjacency matrix $A \in \sR^{N \times N}$ that encodes the causal relations between neuron locations. Finally, we infer the causal circuit.

\subsection{Per-Case Results on \textsc{InterpBench}}\label{app:percase}
Table~\ref{tab:shd_percase} reports the per-case mean and standard deviation of SHD over the three trials summarized in Section~\ref{sec:cd_neuron_exp}. Because \textsc{InterpBench} uses fixed synthetic transformers with known ground-truth circuits, variation between trials arises only from random seed and prompt sampling, and the run-to-run standard deviation is correspondingly small ($0.00$ for all three methods on 11 of the 16 cases). The reported means therefore support a reliable descriptive comparison: \mn{}-linear matches EAP-ig exactly, in both mean and standard deviation, on 11 of the 16 cases, and differs by at most $1$ SHD on the remaining $5$, supporting the ``efficiency at parity of accuracy'' framing as a descriptive comparability claim rather than a claim of statistical equivalence.

Table~\ref{tab:shd_errors} decomposes SHD into false positives (spurious edges) and false negatives (missed edges). False negatives dominate the total error for all three methods, accounting for approximately $81\%$ of \mn{}-linear's total errors, $75\%$ of EAP-ig's, and $72\%$ of EAP's, meaning that all three methods primarily fail by missing ground-truth edges rather than by adding spurious ones. \mn{}-linear has the lowest per-case false-positive rate overall ($0.44$, versus $0.54$ for EAP-ig and $0.85$ for EAP), consistent with its $\ell_1$ sparsity constraint, which discourages spurious edges even in cases where intervention-based baselines occasionally add them; adding spurious edges is generally more problematic in practice than missing edges. The trade-off is a slightly higher false-negative rate than EAP-ig ($1.90$ versus $1.65$ per case): \mn{}-linear misses one or more edges in a few cases (Cases 19, 29, and 35) where EAP-ig recovers them.

\begin{table*}[hpt]
\centering
\begin{tabular}{cccccc}
\toprule
Case & $|E|$ & EAP & EAP-ig & \mn{}-linear & \mn{}-nonlinear \\
\midrule
3  & 4 & $1.00$\scalebox{0.8}{$\pm0.00$}  & $1.00$\scalebox{0.8}{$\pm0.00$}  & $1.00$\scalebox{0.8}{$\pm0.00$} & $1.00$\scalebox{0.8}{$\pm0.00$} \\
4  & 8 & $4.00$\scalebox{0.8}{$\pm1.41$}  & $2.00$\scalebox{0.8}{$\pm0.00$}  & $2.33$\scalebox{0.8}{$\pm0.47$} & $2.00$\scalebox{0.8}{$\pm0.00$} \\
8  & 2 & $1.00$\scalebox{0.8}{$\pm0.00$}  & $1.00$\scalebox{0.8}{$\pm0.00$}  & $1.00$\scalebox{0.8}{$\pm0.00$} & $1.00$\scalebox{0.8}{$\pm0.00$} \\
11 & 3 & $2.33$\scalebox{0.8}{$\pm0.94$}  & $1.00$\scalebox{0.8}{$\pm0.00$}  & $1.00$\scalebox{0.8}{$\pm0.00$} & $1.00$\scalebox{0.8}{$\pm0.00$} \\
13 & 3 & $5.00$\scalebox{0.8}{$\pm0.00$}  & $5.00$\scalebox{0.8}{$\pm0.00$}  & $5.00$\scalebox{0.8}{$\pm0.00$} & $4.00$\scalebox{0.8}{$\pm0.82$} \\
18 & 6 & $5.00$\scalebox{0.8}{$\pm0.00$}  & $3.00$\scalebox{0.8}{$\pm0.82$}  & $3.00$\scalebox{0.8}{$\pm0.00$} & $3.00$\scalebox{0.8}{$\pm0.00$} \\
19 & 3 & $5.00$\scalebox{0.8}{$\pm0.00$}  & $3.00$\scalebox{0.8}{$\pm0.00$}  & $3.33$\scalebox{0.8}{$\pm0.47$} & $3.00$\scalebox{0.8}{$\pm0.00$} \\
20 & 2 & $2.00$\scalebox{0.8}{$\pm0.00$}  & $1.00$\scalebox{0.8}{$\pm0.00$}  & $1.00$\scalebox{0.8}{$\pm0.00$} & $1.00$\scalebox{0.8}{$\pm0.00$} \\
21 & 8 & $13.00$\scalebox{0.8}{$\pm0.82$} & $11.00$\scalebox{0.8}{$\pm0.00$} & $11.00$\scalebox{0.8}{$\pm0.00$} & $10.00$\scalebox{0.8}{$\pm0.00$} \\
26 & 2 & $1.00$\scalebox{0.8}{$\pm0.00$}  & $1.00$\scalebox{0.8}{$\pm0.00$}  & $1.00$\scalebox{0.8}{$\pm0.00$} & $1.00$\scalebox{0.8}{$\pm0.00$} \\
29 & 2 & $2.00$\scalebox{0.8}{$\pm0.00$}  & $1.00$\scalebox{0.8}{$\pm0.00$}  & $2.00$\scalebox{0.8}{$\pm0.00$} & $1.00$\scalebox{0.8}{$\pm0.00$} \\
33 & 2 & $1.00$\scalebox{0.8}{$\pm0.00$}  & $1.00$\scalebox{0.8}{$\pm0.00$}  & $1.00$\scalebox{0.8}{$\pm0.00$} & $1.00$\scalebox{0.8}{$\pm0.00$} \\
34 & 2 & $2.00$\scalebox{0.8}{$\pm0.00$}  & $1.00$\scalebox{0.8}{$\pm0.00$}  & $1.00$\scalebox{0.8}{$\pm0.00$} & $1.00$\scalebox{0.8}{$\pm0.00$} \\
35 & 2 & $3.00$\scalebox{0.8}{$\pm0.00$}  & $1.00$\scalebox{0.8}{$\pm0.00$}  & $2.00$\scalebox{0.8}{$\pm0.00$} & $1.00$\scalebox{0.8}{$\pm0.00$} \\
36 & 2 & $1.00$\scalebox{0.8}{$\pm0.00$}  & $1.00$\scalebox{0.8}{$\pm0.00$}  & $1.00$\scalebox{0.8}{$\pm0.00$} & $1.00$\scalebox{0.8}{$\pm0.00$} \\
37 & 2 & $1.00$\scalebox{0.8}{$\pm0.00$}  & $1.00$\scalebox{0.8}{$\pm0.00$}  & $1.00$\scalebox{0.8}{$\pm0.00$} & $1.00$\scalebox{0.8}{$\pm0.00$} \\
\bottomrule
\end{tabular}
\caption{Per-case SHD (mean $\pm$ standard deviation over three trials) on the 16 main \textsc{InterpBench} synthetic cases. $|E|$ is the number of ground-truth edges. Lower is better.}
\label{tab:shd_percase}
\end{table*}

\begin{table*}[hpt]
\centering
\small
\begin{tabular}{cc cccc cccc cccc}
\toprule
& & \multicolumn{4}{c}{EAP} & \multicolumn{4}{c}{EAP-ig} & \multicolumn{4}{c}{\mn{}-linear} \\
\cmidrule(lr){3-6}\cmidrule(lr){7-10}\cmidrule(lr){11-14}
Case & $|E|$ & FP & FN & Prec. & Rec. & FP & FN & Prec. & Rec. & FP & FN & Prec. & Rec. \\
\midrule
3  & 4 & 0.00 & 1.00 & 1.000 & 0.750 & 0.00 & 1.00 & 1.000 & 0.750 & 0.00 & 1.00 & 1.000 & 0.750 \\
4  & 8 & 1.33 & 2.67 & 0.794 & 0.667 & 0.00 & 2.00 & 1.000 & 0.750 & 0.00 & 2.33 & 1.000 & 0.708 \\
8  & 2 & 0.00 & 1.00 & 1.000 & 0.500 & 0.00 & 1.00 & 1.000 & 0.500 & 0.00 & 1.00 & 1.000 & 0.500 \\
11 & 3 & 0.67 & 1.67 & 0.667 & 0.444 & 0.00 & 1.00 & 1.000 & 0.667 & 0.00 & 1.00 & 1.000 & 0.667 \\
13 & 3 & 2.00 & 3.00 & 0.000 & 0.000 & 2.00 & 3.00 & 0.000 & 0.000 & 2.00 & 3.00 & 0.000 & 0.000 \\
18 & 6 & 1.00 & 4.00 & 0.667 & 0.333 & 0.67 & 2.33 & 0.850 & 0.611 & 0.00 & 3.00 & 1.000 & 0.500 \\
19 & 3 & 2.00 & 3.00 & 0.000 & 0.000 & 1.00 & 2.00 & 0.500 & 0.333 & 0.00 & 3.00 & 0.000 & 0.000 \\
20 & 2 & 0.00 & 2.00 & 0.000 & 0.000 & 0.00 & 1.00 & 1.000 & 0.500 & 0.00 & 1.00 & 1.000 & 0.500 \\
21 & 8 & 5.67 & 7.33 & 0.103 & 0.083 & 5.00 & 6.00 & 0.286 & 0.250 & 5.00 & 6.00 & 0.286 & 0.250 \\
26 & 2 & 0.00 & 1.00 & 1.000 & 0.500 & 0.00 & 1.00 & 1.000 & 0.500 & 0.00 & 1.00 & 1.000 & 0.500 \\
29 & 2 & 0.00 & 2.00 & 0.000 & 0.000 & 0.00 & 1.00 & 1.000 & 0.500 & 0.00 & 2.00 & 0.000 & 0.000 \\
33 & 2 & 0.00 & 1.00 & 1.000 & 0.500 & 0.00 & 1.00 & 1.000 & 0.500 & 0.00 & 1.00 & 1.000 & 0.500 \\
34 & 2 & 0.00 & 2.00 & 0.000 & 0.000 & 0.00 & 1.00 & 1.000 & 0.500 & 0.00 & 1.00 & 1.000 & 0.500 \\
35 & 2 & 1.00 & 2.00 & 0.000 & 0.000 & 0.00 & 1.00 & 1.000 & 0.500 & 0.00 & 2.00 & 0.000 & 0.000 \\
36 & 2 & 0.00 & 1.00 & 1.000 & 0.500 & 0.00 & 1.00 & 1.000 & 0.500 & 0.00 & 1.00 & 1.000 & 0.500 \\
37 & 2 & 0.00 & 1.00 & 1.000 & 0.500 & 0.00 & 1.00 & 1.000 & 0.500 & 0.00 & 1.00 & 1.000 & 0.500 \\
\midrule
\multicolumn{2}{c}{Overall} & 0.85 & 2.23 & 0.514 & 0.299 & 0.54 & 1.65 & 0.852 & 0.491 & 0.44 & 1.90 & 0.705 & 0.398 \\
\bottomrule
\end{tabular}
\caption{Per-case error breakdown on the 16 main \textsc{InterpBench} synthetic cases, averaged over three trials. FP and FN are the per-case counts of spurious and missed edges; $|E|$ is the number of ground-truth edges. False negatives dominate the total error for all three methods, and \mn{}-linear has the lowest false-positive rate.}
\label{tab:shd_errors}
\end{table*}

\subsection{\mn{} on Sparse Autoencoder Features}\label{app:sae_features} 
We extract GPT-2 small's neuron activations on the $M$ training sentences.
We select the final outputs from each transformer block (layer) as our locations of interest. Given a prompt with $n_{\text{token}}$ tokens, we obtain transformer outputs at the $i^{\text{th}}$ layer with shape $d\times n_{\text{token}}$ and the corresponding sparse autoencoder features with shape $D\times n_{\text{token}}$. We then collect sparse autoencoder features for all $M$ prompts and average across tokens to produce sparse feature activations $\mZ_i \in \sR^{D \times M}$. Repeating this across all $L$ layers yields our dataset $\{\mZ_1, \mZ_2, \dots, \mZ_L\}$. In our setting, $L = 12$, $d = 768$, $D = 32{,}768$, and $M = 8{,}551$. For the CoLA task, the prediction target $y$ indicates whether a sentence is linguistically acceptable. Substituting $\{\mZ_i\}_{i=1}^L$ and $y$ into Eq.~(\ref{eq:layer2layer_obj}) and Eq.~(\ref{eq:layer2y_obj}), we learn weighted adjacency matrices between consecutive layers, $\{A_{i,i+1}\}_{i=1}^{L-1}$, and the weighted adjacency matrix between the final layer sparse autoencoder features $\vz_L$ and the prediction target $y$, i.e., $A_{L,y}$.

We expand on the procedure summarised in Section~\ref{sec:interp}.

\textbf{Identifying essential features.}
Starting with the final-layer adjacency $A_{L,y}$, we select features in $\vz_L \in \sR^D$ that are important for predicting $y$. The measure of importance can be either the absolute coefficients $|A_{L,y}|$, or, when prompt-specific analysis is desired, the Hadamard product $\vs = |A_{L,y}| \odot |\vz_L| \in \sR^D$ between the absolute coefficients and the absolute activations of $\vz_L$ for the chosen prompt.

\textbf{Multi-prompt approach.}
Given an important feature $z_{L,k}$, we collect multiple prompts together with the tokens that strongly activate the feature, then infer the encoded semantic concept by inspecting the collected tokens. For example, words ending in ``-self'' consistently activate $z_{12,20726}$, suggesting that this feature captures the presence of such words. Tables~\ref{tab:semantic_multi_20726}, \ref{tab:semantic_multi_776}, and \ref{tab:semantic_multi_19322} provide additional examples.

\begin{table*}[hpt]
    \centering
    \setlength{\tabcolsep}{3pt}
    \begin{tabular}{lcl}
    \toprule
    Prompts & Values of $z_{12,20726}$ & Per-token values \\
    \midrule
    Kiss \yhlt{himself}. & 0.3051 & (0, 0, 0, \yhlt{1.5254}, 0) \\
    This movie just watches \yhlt{itself}. & 0.1861 & (0, 0, 0, 0, 0, \yhlt{1.3027}, 0) \\
    This window just opens \yhlt{itself}. & 0.1759 & (0, 0, 0, 0, 0, \yhlt{1.2311}, 0) \\
    This list includes my name on \yhlt{itself}. & 0.1697 & (0, 0, 0, 0, 0, 0, 0, \yhlt{1.5277}, 0) \\
    This silver polishes \yhlt{itself}. & 0.1692 & (0, 0, 0, 0, 0, \yhlt{1.1844}, 0) \\
    He said that \yhlt{himself} was hungry. & 0.1689 & (0, 0, 0, 0, \yhlt{1.3509}, 0, 0, 0) \\
    Every picture of \yhlt{itself} arrived. & 0.1665 & (0, 0, 0, 0, \yhlt{1.1652}, 0, 0) \\
    Bill understands Mary and \yhlt{himself}. & 0.1638 & (0, 0, 0, 0, 0, \yhlt{1.1467}, 0) \\
    \yhlt{Myself} saw me. & 0.1602 & (0, 0, \yhlt{0.8009}, 0, 0) \\
    \bottomrule
    \end{tabular}
    \caption{\textit{Multi-prompt} approach for identifying semantic concepts for sparse feature $z_{12,20726}$.}
    \label{tab:semantic_multi_20726}
\end{table*}

\begin{table*}[hpt]
    \centering
    \setlength{\tabcolsep}{3pt}
    \begin{tabular}{ll}
    \toprule
    Prompts & Per-token values of $z_{12,776}$ \\
    \midrule
    \yhlt{Hun}\ghlt{ger} fainted Sharon. & (0, \yhlt{0.9980}, \ghlt{3.3986}, 0, 0, 0, 0) \\
    Many people were dying of \yhlt{thirst}. & (0, 0, 0, 0, 0, 0, \yhlt{2.0910}, 0) \\
    One people was dying of \yhlt{thirst}. & (0, 0, 0, 0, 0, 0, \yhlt{1.8140}, 0) \\
    John whined that he was \yhlt{hungry}. & (0, 0, 0, 0, 0, 0, 0, 0, \yhlt{1.9004}, 0) \\
    Many soldiers have claimed bottled water satisfies \yhlt{thirst} best. & (0, 0, 0, 0, 0, 0, 0, 0, \yhlt{1.9243}, 0, 0) \\
    \bottomrule
    \end{tabular}
    \caption{\textit{Multi-prompt} approach for identifying semantic concepts for sparse feature $z_{12,776}$.}
    \label{tab:semantic_multi_776}
\end{table*}

\begin{table*}[hpt]
    \centering
    \setlength{\tabcolsep}{3pt}
    \begin{tabular}{ll}
    \toprule
    Prompts & Per-token values of $z_{12,19322}$ \\
    \midrule
    The teacher became \yhlt{tired} \ghlt{of} the students. & (0, 0, 0, 0, \yhlt{2.7252}, \ghlt{1.0383}, 0, 0, 0) \\
    The president looked \yhlt{weary}. & (0, 0, 0, 0, \yhlt{2.1888}, 0) \\
    Genie intoned that she was \yhlt{tired}. & (0, 0, 0, 0, 0, 0, 0, 0, \yhlt{2.6658}) \\
    John placed him \yhlt{busy}. & (0, 0, 0, 0, \yhlt{1.5510}, 0) \\
    Visiting relatives can be \yhlt{boring}. & (0, 0, 0, 0, 0, 0, \yhlt{1.8287}, 0) \\
    \bottomrule
    \end{tabular}
    \caption{\textit{Multi-prompt} approach for identifying semantic concepts for sparse feature $z_{12,19322}$.}
    \label{tab:semantic_multi_19322}
\end{table*}

\textbf{Single-prompt approach.}
To validate the plausibility of the inferred semantic concept, we choose a single prompt, systematically perturb one or more tokens consistent with that concept, and observe the change in $z_{L,k}$. If the perturbation deactivates the feature, the concept is considered reasonable. For example, in the prompt ``He said that himself was hungry,'' the word ``himself'' (the fifth token) activates $z_{12,20726}$ to $1.3509$; replacing it with ``him,'' ``he,'' or other alternatives that lack the ``-self'' suffix drops the activation to $0$. Table~\ref{tab:semantic_multi_3092} provides an additional example.

\begin{table*}[hpt]
    \centering
    \setlength{\tabcolsep}{3pt}
    \begin{tabular}{ll}
    \toprule
    Prompts & Per-token values of $z_{12,3092}$ \\
    \midrule
    He said that himself was hungry\yhlt{.} & (0, 0, 0, 0, 0, 0, 0, \yhlt{6.0688}) \\
    He said that himself was hungry\yhlt{yet} & (0, 0, 0, 0, 0, 0, 0, \yhlt{0}) \\
    He said that himself was hungry\yhlt{,} & (0, 0, 0, 0, 0, 0, 0, \yhlt{0}) \\
    He said that himself was hungry\yhlt{?} & (0, 0, 0, 0, 0, 0, 0, \yhlt{3.4722}) \\
    He said that himself was hungry\yhlt{!} & (0, 0, 0, 0, 0, 0, 0, \yhlt{3.7391}) \\
    \bottomrule
    \end{tabular}
    \caption{\textit{Single-prompt} approach for identifying semantic concepts for sparse feature $z_{12,3092}$.}
    \label{tab:semantic_multi_3092}
\end{table*}

\textbf{Backward tracing across layers.}
Applying the above to the final layer yields semantic labels for the most important features ($z_{12,20726}$ for ``-self,'' $z_{12,3092}$ for ``ending punctuation,'' $z_{12,776}$ for ``thirst/hunger,'' $z_{12,19322}$ for ``tired/weary,'' etc.). For each important final-layer feature $z_{L,k}$, we trace its most influential parents in layer $L{-}1$ using the learned adjacency $A_{L-1,L}$; as in the final layer, we may instead use the prompt-specific importance $\vs = |\partial z_{L,k}/\partial \vz_{L-1}| \odot |\vz_{L-1}| \in \sR^D$. The top parents (e.g.\ $z_{11,6368}$, $z_{11,29778}$, $z_{11,29041}$, $z_{11,21518}$) are then labelled by the same multi- and single-prompt procedures. Repeating this across consecutive layer pairs yields tree-shaped circuit paths spanning the transformer blocks.

\subsection{\mn{} for Domain Generalization}\label{app:dg} 
We select SAE features from a specific location of a pre-trained LLM, such as the transformer output at layer $22$ in Gemma-2-2B.
For each prompt, we average the $D$-dimensional SAE features at this location over tokens and collect them across all $M$ prompts in the training data, producing $\mZ_s \in \sR^{D \times M}$. Substituting $\mZ_s$ and the target observations $y = (y_1, y_2, \cdots, y_M) \in \sR^{M}$ into Eq.~(\ref{eq:layer2y_obj}), we estimate the weighted adjacency matrix $A_{s,y} \in \sR^{D}$ and identify important features with large absolute coefficients. This process is
equivalent to training a sparse linear classifier on the SAE features, which we later exploit for profession prediction. From $A_{s,y}$, we further interpret the semantics of the important features using our proposed \textit{multi-prompt} and \textit{single-prompt} approaches (Section~\ref{sec:interp}) and manually identify spurious features associated with gender. Unlike the \textsc{shift} method of \citet{markssparse}, which \emph{decodes} SAE features into
neuron activations after ablation, we ablate spurious features by setting their values to zero and \emph{directly feed} the resulting SAE feature values into our trained linear classifier. In addition, similar to \textsc{shift}, we also investigate retraining the linear classifier on the ablated SAE features.

\section{Interpretability on Sparse Features Circuits}

\subsection{Preliminary of Sparse Autoencoders}\label{app:sae}
\citet{huben2024sparse} uses a ReLU activation with $L_1$ sparsity regularization. Subsequent work explores alternative activation functions $\sigma(\cdot)$ to extract desired SAE features. \citet{rajamanoharan2024jumping} introduces a threshold to determine the minimum pre-activation for feature activation, while \citet{gao2024scaling} and \citet{bussmann2024batchtopk} enforce sparsity by selecting the top $K$ features. \citet{dunefsky2024transcoders} proposes Transcoders, which are similar to SAEs, but focusing on training interpretable approximations of MLPs.
In this work, we employ the following pre-trained LLMs and SAEs:
\begin{itemize}
    \item The open-source GPT-2 small SAEs~\citep{gao2024scaling} for all sublayers of the open-weights GPT-2 small model~\citep{radford2019language}. These SAEs use a ReLU-linear encoder with $D=32768$ and $L_1$ sparsity regularization.
    \item The open-source Pythia-70M SAEs~\citep{markssparse} for all sublayers of the open-weight Pythia-70M~\citep{biderman2023pythia}. These SAEs use a ReLU-linear encoder with $D=64 \times d$ and $L_1$ sparsity regularization.
    \item The open-source Gemma Scope SAEs~\citep{lieberum2024gemma} for all sublayers of the open-weights Gemma-2-2B, Gemma-2-9B models~\citep{gemma_2024}. These SAEs use the JumpReLU linear encoder and set $D = 8\times d$.
\end{itemize}

\subsection{Ablation Study of Sparsity Constraint}
Table~\ref{tab:ablation_lambda} presents an ablation study on the sparsity coefficient $\lambda$ for circuit discovery between the last layer sparse features and prediction target. When $\lambda = 0$, the model achieves the highest training accuracy ($99.36\%$) but suffers severe overfitting, as reflected in a large performance drop on the test set ($71.10\%$). Introducing a small sparsity constraint ($\lambda = 10^{-5}$) improves test accuracy to $72.77\%$, the best among all settings, indicating enhanced generalization. Larger values of $\lambda$ further enforce sparsity but lead to higher training loss and a notable decline in both training and test accuracy, suggesting that excessive sparsity harms the model's capacity to capture meaningful circuit structure. We therefore select $\lambda = 10^{-5}$ as the optimal setting, as it achieves the best test accuracy and generalization ability while preserving a sparse, interpretable circuit structure.

\begin{table}[hpt]
    \centering
    \scriptsize
    \setlength{\tabcolsep}{3pt}
    \begin{tabular}{cccccc}
    \toprule
    \multirow{2}{*}{$\lambda$} & \multicolumn{3}{c}{Training Set} & \multicolumn{2}{c}{Test Set} \\
    & Pred. Loss & L1 Loss & Pred. Acc. (\%) & Pred. Loss & Pred. Acc. (\%)\\
    \midrule
    0 & 0.0612 & - & \bf{99.36} & 0.9977 & 71.10 \\
    $10^{-5}$ & 0.1741 & 9257.6 & 95.70 & 0.6684 & \bf{72.77} \\
    $5\times 10^{-5}$ & 0.3498 & 1898.7 & 85.44 & 0.5566 & 72.48\\
    $5\times 10^{-4}$ & 0.5535 & 45.3 & 72.14 & 0.5566 & 70.28 \\
    $10^{-4}$ & 0.6143 & 0.5 & 70.44 & 0.6283 & 69.14 \\
    \bottomrule
    \end{tabular}
    \caption{Ablation study of sparsity constraint coefficient $\lambda$ for circuit discovery between last layer $\vz_L$ and prediction target $y$.}
    \label{tab:ablation_lambda}
\end{table}

\subsection{Statistics of Learned Circuit Weighted Adjacency Matrices}
Figure~\ref{fig:adjacency_distribution} examines the coefficient distribution $|A_{L,y}|$ under the best setting ($\lambda = 10^{-5}$). We observe that most coefficient values are extremely small, with $82.2\%$ below 0.01, $83.3\%$ below 0.10, $86\%$ below 0.41, suggesting that only a small subset of features contribute substantially to the prediction. The top 5 essential features clearly dominate the distribution, highlighting the effectiveness of the sparsity constraint in filtering out irrelevant features and isolating semantically interpretable ones.
\begin{figure}[hpt]
    \centering
    \includegraphics[width=\columnwidth]{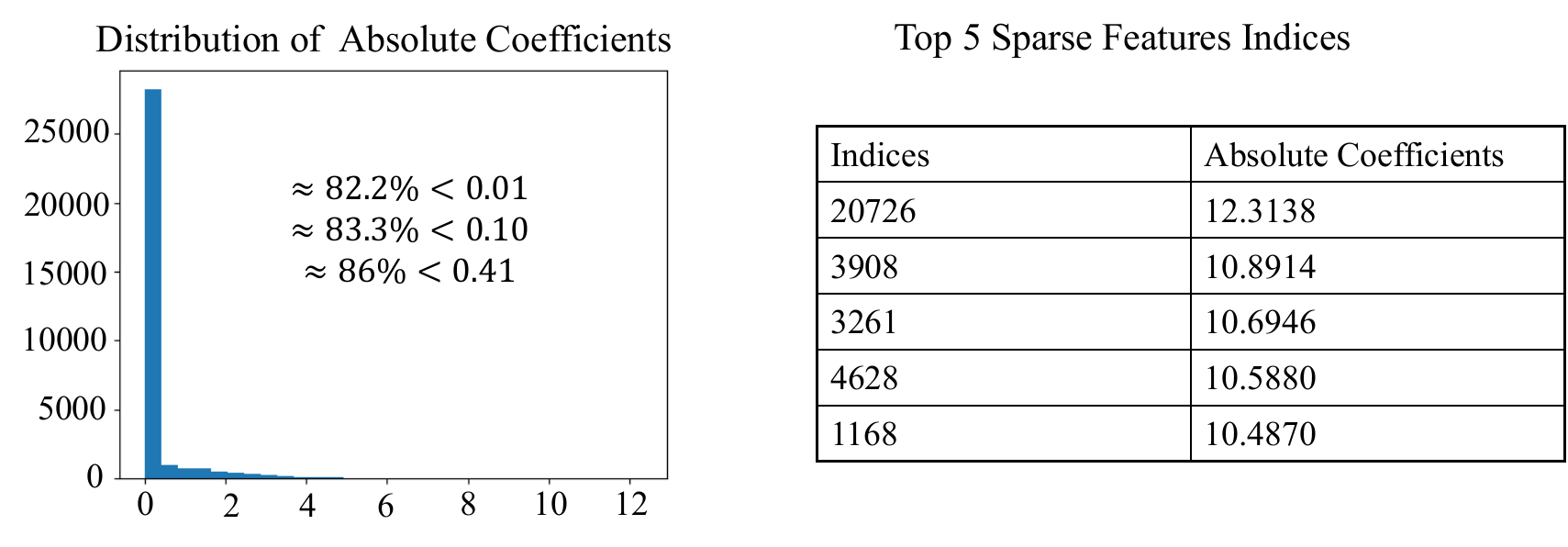}
    \caption{The distribution of $|A_{L,y}|$ and the selected top 5 essential features.}
    \label{fig:adjacency_distribution}
\end{figure}

\begin{figure*}[hpt]
    \centering
    \includegraphics[width=1\textwidth]{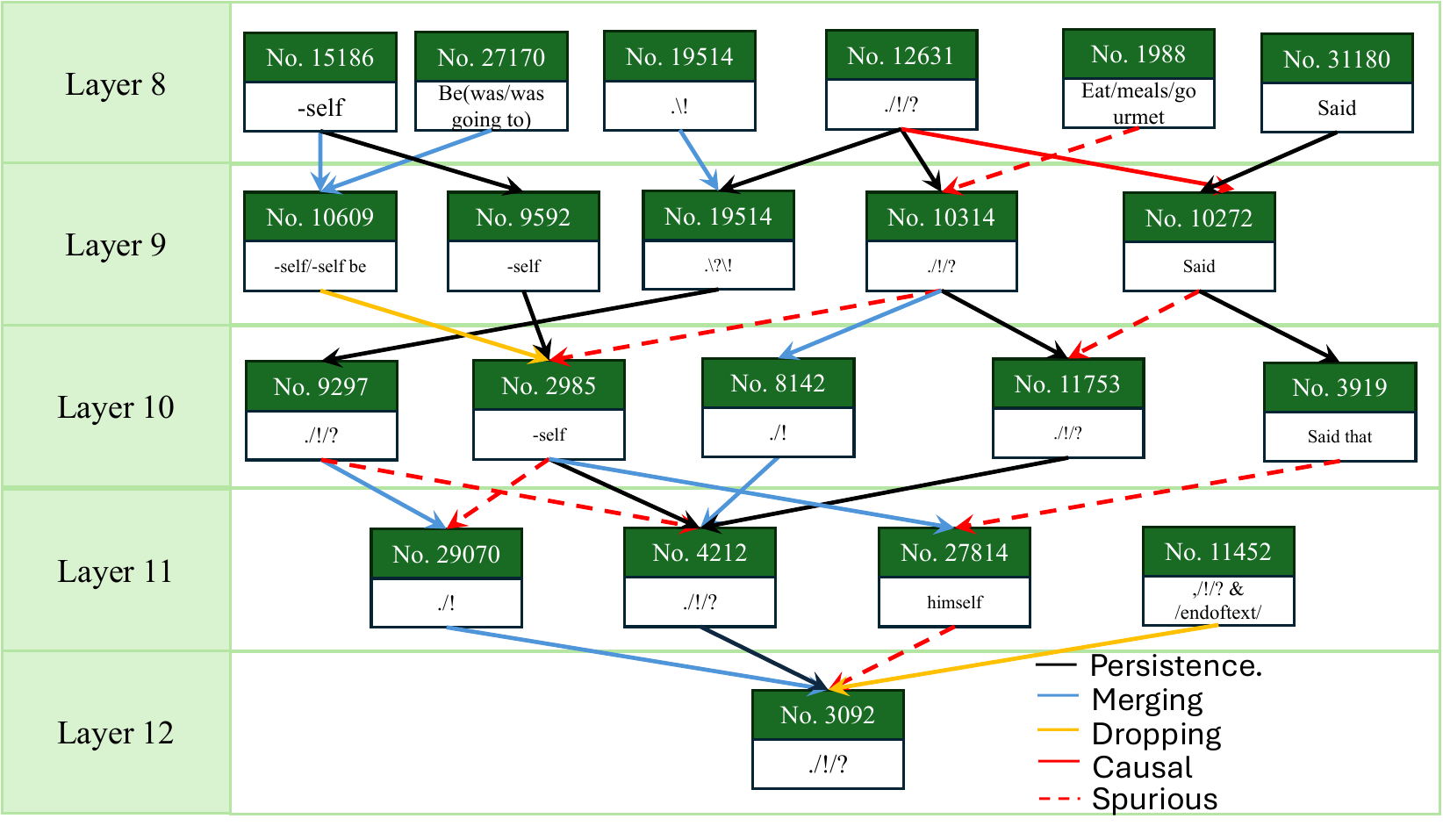}
    \caption{The learned circuit over SAE features on GPT-2 small model. Starting with feature No.\ 3092 in layer 12.}
    \label{fig:circuit_path_2}
\end{figure*}
\subsection{More Examples of Sparse Feature Circuits}\label{app:sparse_features_circuit}

\subsection{Faithfulness and Completeness}\label{app:faithfulness}

We expand on the summary in Section~\ref{sec:interp}.

\textbf{Setup.}
To more comprehensively evaluate the quality of our learned SAE-feature circuits on CoLA, we use the \emph{faithfulness} and \emph{completeness} metrics, following the standard protocol of \citet{markssparse}.  Let the learned circuit be $C$, and define the model output as $m = p(Y=\text{grammatically correct}) - p(Y=\text{grammatically wrong})$.

\textbf{Node ablation (standard, comparable to \textsc{shift}).}
We first apply the standard feature-ablation protocol of \citet{markssparse} and compare against the intervention-based circuit learning method \textsc{shift}.  To ensure fairness, we exclude SAE reconstruction errors and attention/MLP SAEs from the \textsc{shift} setting.  For \mn{}, we focus only on SAE features within the learned circuit and still ablate features in the original LLM.  The top two plots of Figure~\ref{fig:fc} show the node-ablation results: our learned circuit achieves performance comparable to \textsc{shift}, consistent with the finding of \citet{markssparse} that relatively small feature circuits can explain a substantial portion of a model's behavior.

\textbf{Edge ablation (novel, not supported by \textsc{shift}).}
Because our sparse-regression framework explicitly learns edge coefficients $\widehat{A}_{i,j}$, each representing the direct influence of a parent SAE feature on a child node, we can ablate a specific edge by setting its coefficient to zero, thereby removing that dependence within our linear surrogate $\widehat{m} = \widehat{A}_{L,y}^\top \vz_L$.\footnote{We do not apply edge ablation to \textsc{shift}, which does not provide edge-level correlations between SAE features and the output.}  The bottom two plots of Figure~\ref{fig:fc} report these results.  The conclusions mirror those of node ablation: a small subset of essential edges, together with their corresponding SAE features, governs the model's prediction behavior.
\subsection{Bias-in-Bios: Full Setup and Results}\label{app:bib_full}

We expand on the summary in Section~\ref{sec:dg}.

\textbf{Data and task.}
The Bias-in-Bios dataset (BiB)~\citep{de2019bias} consists of professional biographies with the task of classifying an individual's profession.  The dataset encodes a spurious semantic concept, gender, based on which two subsets are constructed: an \emph{ambiguous} set, where profession and gender are strongly correlated (e.g., all professors are assumed to be male, while nurses are assumed to be female), and a \emph{balanced} set, where profession and gender are independent (equal numbers of male professors, male nurses, female professors, and female nurses).  The goal is to produce a profession classifier that performs accurately on the balanced set but is not biased by the spurious gender signal in the ambiguous set.  While \citet{markssparse} demonstrate their approach on Pythia-70M and Gemma-2-2B, we extend their evaluation to Gemma-2-9B.

\textbf{Baselines.}
We compare against:
(i)~\textsc{original}, a predictor trained directly on the ambiguous set;
(ii)~\textsc{oracle}, a predictor trained on the balanced set, included as a (non-comparable) upper bound;
(iii)~\textsc{cbp}, concept bottleneck probing~\citep{yan2023robust};
(iv)~\textsc{shift}, spurious human-interpretable feature trimming~\citep{markssparse};
(v)~\textsc{shift}-retrain, the variant in which the linear classifier is retrained on the ablated features;
and
(vi)~a \emph{linear probing}~\citep{gurnee2023finding} trained directly on the same SAE features with zero ablation to eliminate the irrelevant features.
For \textsc{shift}, we use the SAE-feature variant with manual inspection and exclude the neuron-level and unsupervised variants due to their consistently inferior performance reported in the original paper.

\textbf{\mn{} for domain generalization.}
We select SAE features from a single layer of each LLM (e.g., the transformer output at layer~22 of Gemma-2-2B).  For each prompt we average the $D$-dimensional SAE features over tokens, producing $\mZ_s \in \sR^{D \times M}$.  Substituting $\mZ_s$ and the target $\vy \in \sR^M$ into Eq.~(\ref{eq:layer2y_obj}) estimates the weighted importance vector $\widehat{A}_{s,y} \in \sR^D$, which is equivalent to training a sparse linear classifier on the SAE features.  We rank features by $|\widehat{A}_{s,y}|$, manually identify gender-correlated features via the multi-prompt and single-prompt procedures, set their values to zero, and feed the result either directly into the trained classifier (\mn{}) or into a freshly retrained linear classifier (\mn{}-retrain).

\textbf{Full results.}
Table~\ref{tab:bib_pythia_gemma} reports profession accuracy, gender leakage (closer to 50 is better), and worst-group accuracy across the three LLMs.  \mn{} and \mn{}-retrain achieve profession accuracies comparable to, and in some cases slightly better than, \textsc{shift} and \textsc{shift}-retrain; gender-leakage values stay near the 50\% balanced target; and worst-group accuracies match or exceed the strongest non-\textsc{oracle} baseline.  Efficiency results (Table~\ref{tab:bib_pythia_gemma_eff}) further underscore the strengths of \mn{}: it requires fewer features and substantially less runtime than \textsc{shift}, particularly for large models.

\end{document}